\documentclass{article}

\usepackage[preprint]{neurips_2026}
\usepackage{graphicx} 
\usepackage{makecell}
\usepackage{booktabs}
\usepackage{multirow}
\usepackage{amsmath}
\usepackage{amssymb}
\usepackage{subcaption}
\newtheorem{definition}{Definition}
\newtheorem{lemma}{Lemma}
\newtheorem{proposition}{Proposition}
\newtheorem{proof}{Proof}

\usepackage[utf8]{inputenc} 
\usepackage[T1]{fontenc}    
\usepackage{hyperref}       
\usepackage{url}            
\usepackage{booktabs}       
\usepackage{amsfonts}       
\usepackage{nicefrac}       
\usepackage{microtype}      
\usepackage{xcolor}         

\title{LAQuant: A Simple Overhead-free Large Reasoning Model Quantization by Layer-wise Lookahead Loss}

\author{%
  Euntae Choi\thanks{Equal contribution.} \\
  Seoul National University \\
  \texttt{euntae.choi175@gmail.com} \\
  \And
  Sumin Song\footnotemark[1] \\
  Seoul National University \\
  \texttt{songsm921@snu.ac.kr} \\
  \And
  Sungjoo Yoo\thanks{Corresponding author.} \\
  Seoul National University \\
  \texttt{sungjoo.yoo@gmail.com} \\
}

\begin{document}

\maketitle

\begin{abstract}
Large reasoning models (LRMs) reach competition-level math and coding accuracy via long autoregressive decoding, making per-token decoding cost a primary deployment concern. Weight quantization is the standard tool for acceleration, but representative recipes --- including state-of-the-art end-to-end (E2E) QAT --- lose accuracy on long-decoding reasoning benchmarks despite preserving perplexity and short-decode accuracy. Through a systematic gradient-direction analysis, we identify two factors driving this gap: (i) KV-cache fidelity preservation under the QAT loss, which E2E supervision attenuates via the softmax Fisher metric; and (ii) Hessian-subspace alignment between calibration data and the deployment distribution. We propose LookAhead Quantization (LAQuant), a layer-wise weight-only QAT method that addresses both factors without online-transform overhead by combining reasoning-domain calibration with a one-layer lookahead loss whose implicit cross-layer co-adaptation preserves the next-layer residual stream. For Qwen3-4B under W3G128 quantization, LAQuant improves AIME25 Pass@1 over ParoQuant by 15.11\,pp (1.93\,pp over ParoQuant++ at matched calibration) while achieving a 3.42$\times$ decoding speedup over FP16 on RTX A6000, compared with ParoQuant's 3.01$\times$. Code will be made available at \url{blind_review}.
\end{abstract}

\section{Introduction}
\label{sec:intro}
Large reasoning models (LRMs)~\cite{qwen3,grpo,deepseekr1}, powered by reinforcement learning~\cite{grpo,dapo} and test-time scaling~\cite{tts}, have recently achieved strong performance on challenging STEM and coding tasks. Since accuracy scales with reasoning-trace length, per-token decoding efficiency is a critical deployment concern.

Quantization is a common and essential technique for efficient LLM inference~\cite{atom,qserve,quarot,gptq,awq,omniquant,qtip,vptq}, and its importance is amplified under the long-decoding workloads of LRMs. For conventional LLMs, quantization algorithms have advanced significantly: some methods push weight-only quantization to 2-bit~\cite{bitdistiller,efficientqat,dbllm,vptq} or even binary~\cite{billm}, while others jointly quantize weights, activations, and the KV cache at 4-bit~\cite{quarot,spinquant}. On competitive reasoning benchmarks~\cite{aime2024,aime2025,gpqa,lcb} that require long autoregressive decoding, however, representative methods such as AWQ~\cite{awq}, GPTQ~\cite{gptq}, and EfficientQAT~\cite{efficientqat} are reported to exhibit visible degradation even under 4-bit weight-only quantization, indicating that success on conventional benchmarks does not transfer to long-decode reasoning settings.

In response, recent works have proposed quantization methods specifically targeting LRMs. ParoQuant~\cite{paroquant} combines minimal pairwise rotation with channel scaling to reduce per-layer quantization error, but its non-absorbable rotation introduces per-token inference overhead, which is non-trivial given the long thinking-token budgets typical of LRM deployment. ReasoningQAT~\cite{reasoningqat} chains layer-wise quantization-aware training (QAT) with end-to-end (E2E) QAT under a domain-mixed calibration corpus, improving over post-training methods but still leaving headroom on competition-level benchmarks. Despite these advances, the question of why LRM quantization is more difficult than its conventional counterpart remains open.

The E2E QAT objective optimizes only the final logit distribution, placing no direct constraint on per-layer KV-cache errors --- yet KV-cache fidelity is precisely what long-decoding quality requires, since each generated token re-attends to every prior K and V projection. A gradient-direction analysis identifies two factors driving this difficulty: (i) KV-cache fidelity preservation under the QAT loss, and (ii) Hessian-subspace alignment between calibration data and deployment. We propose \textbf{L}ook\textbf{A}head \textbf{Quant}ization (LAQuant), a layer-wise weight-only QAT addressing both with zero online-transform overhead. We focus on W3G128, where per-layer error and KV-cache compounding are most severe; the analysis and method extend to W4G128. LAQuant pairs a one-layer lookahead loss with a Hessian-aligned calibration corpus: the lookahead loss is an implicit Jacobian-weighted error attribution that couples the seven trainable projections of each layer through the next-layer output, producing cross-layer co-adaptation; the calibration corpus keeps the optimization target faithful to the deployment subspace. Our contributions are summarized as follows:
\begin{itemize}
    \item \textbf{A gradient-direction framework} for LRM weight-only QAT that identifies two factors driving the quantization gap: end-to-end supervision attenuates per-layer learning signals (including for KV-cache directions) via the softmax Fisher metric, and calibration-data choice is governed by activation-Hessian subspace alignment with the deployment distribution.
    \item \textbf{LAQuant}, a unified layer-wise QAT recipe that combines lookahead-based loss reshaping with Hessian-aligned calibration data to jointly address both factors, without any non-absorbable online transformation.
    \item \textbf{State-of-the-art empirical results} on challenging STEM reasoning benchmarks, with complementary improvements over the two leading paradigms: a \textbf{19.38 pp} AIME25 Pass@1 gain over ReasoningQAT on Qwen3-4B W3G128, and a \textbf{3.42$\times$} decoding speedup over FP16 on RTX A6000 GPU (vs ParoQuant's 3.01$\times$), establishing a new accuracy--speed Pareto frontier for weight-only quantization of large reasoning models.
\end{itemize}
\vspace{-1.0em}

\begin{figure}[h]
  \centering
  \includegraphics[width=\linewidth]{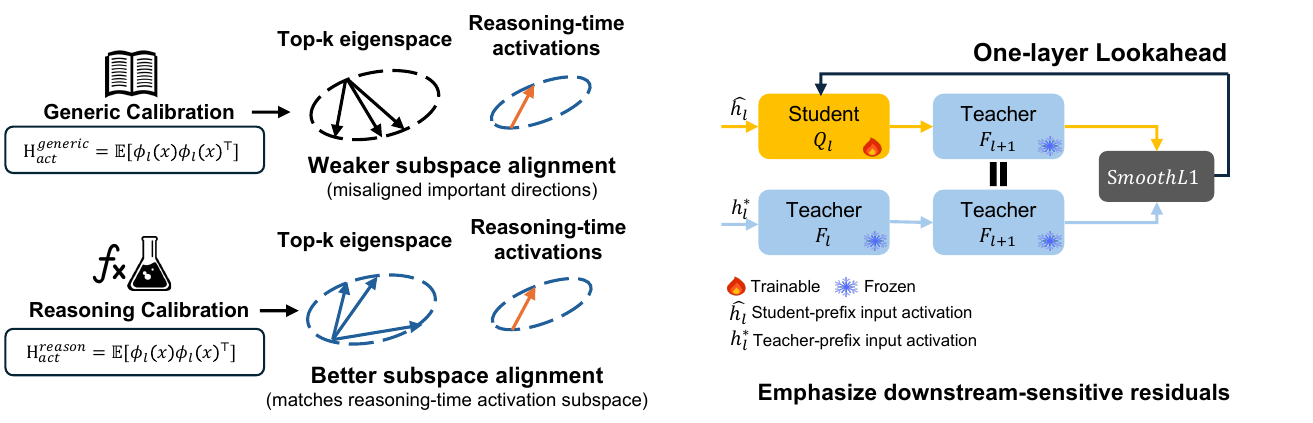}
  \caption{Overview of \textbf{LAQuant}. We combine a one-layer lookahead loss (Section~\ref{sec:lookahead}) with reasoning-domain calibration data (Section~\ref{sec:calibdata}) in a layer-wise QAT pipeline that produces standard weight-only quantized models with no online transformations.}
  \label{fig:overview}
\end{figure}

\section{Preliminaries}

\paragraph{LLM Weight Quantization.}
Linear weight quantization maps a full precision weight $w$ to $b$-bit integer indices $q$, scales $s$, and zeros $z$ as follows:
\begin{equation}
  q \;=\; \mathrm{clip}\!\left(\left\lfloor\tfrac{w}{s}\right\rceil + z,\; 0,\; 2^b - 1\right),
  \qquad
  \hat w \;=\; s \cdot (q - z),
  \label{eq:quant}
\end{equation}
where $\lfloor\cdot\rceil$ denotes round-to-nearest and $\hat w$ is the dequantized weight. Grouped quantization computes the scales and zeros on each consecutive $g$ elements in the rows of the weight, which is denoted W$b$G$g$ (e.g., W3G128 for 3-bit weights with group size 128).

\paragraph{Layer-wise QAT.}
Layer-wise QAT~\cite{efficientqat} optimizes the transformer block-by-block, minimizing the reconstruction error between the FP16 teacher's layer output and the quantized student's layer output. The learnable parameters are the per-layer quantization parameters (scales and zeros) and optionally the weights, clipping bounds~\cite{omniquant}, and channel scales~\cite{smoothquant,awq}; gradients pass through the rounding operation via the straight-through estimator. A full related-work survey appears in Appendix~\ref{app:related_work}.

\section{Motivation}
\label{sec:motivation}

\begin{figure}[!t]
  \centering
  \begin{subfigure}[b]{0.61\linewidth}
    \centering
    \includegraphics[width=\linewidth]{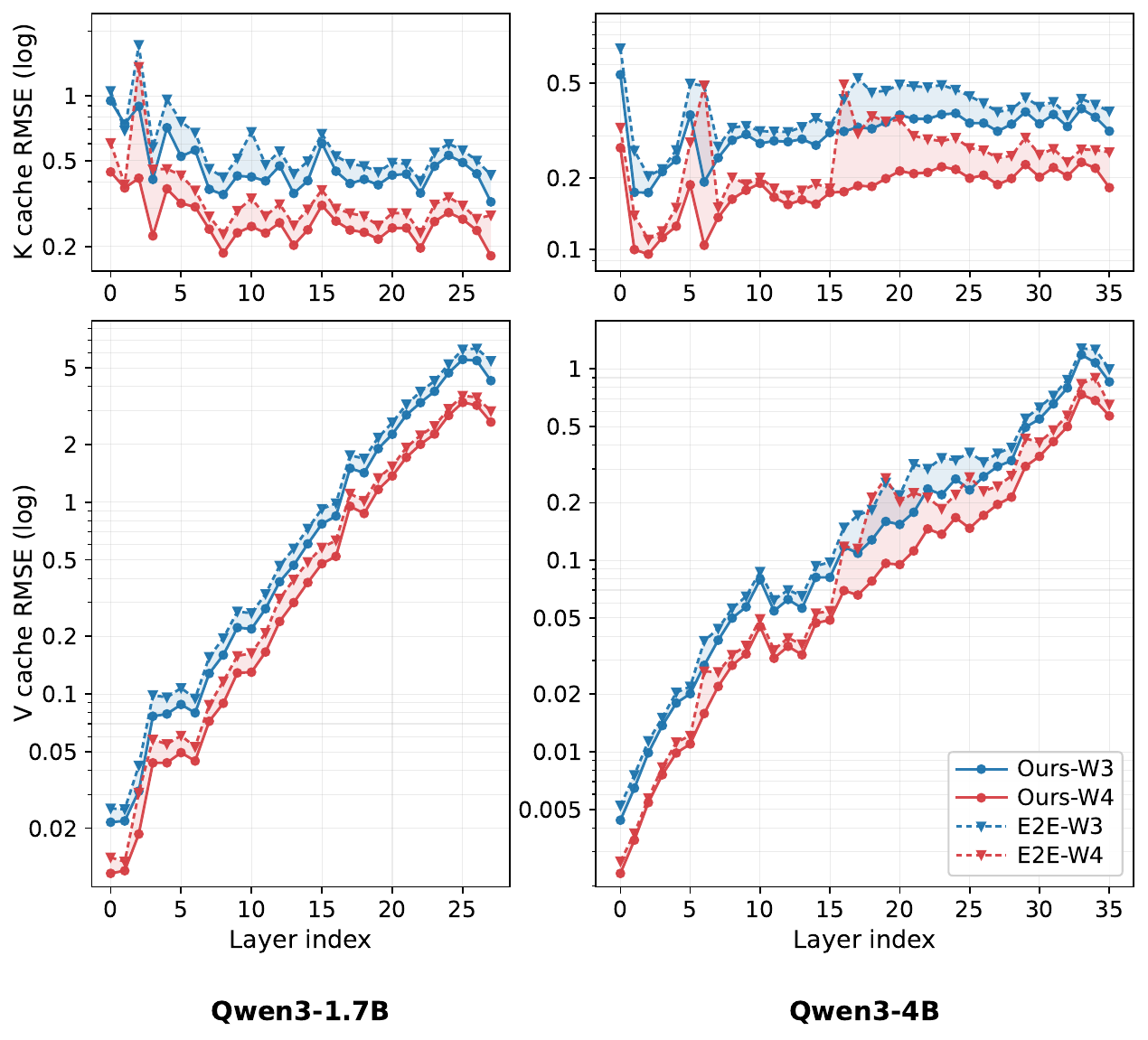}
    \caption{Per-layer K- and V-cache RMSE on AIME25 reasoning traces.}
    \label{fig:kv_aime}
  \end{subfigure}\hfill
  \begin{subfigure}[b]{0.37\linewidth}
    \centering
    \includegraphics[width=\linewidth]{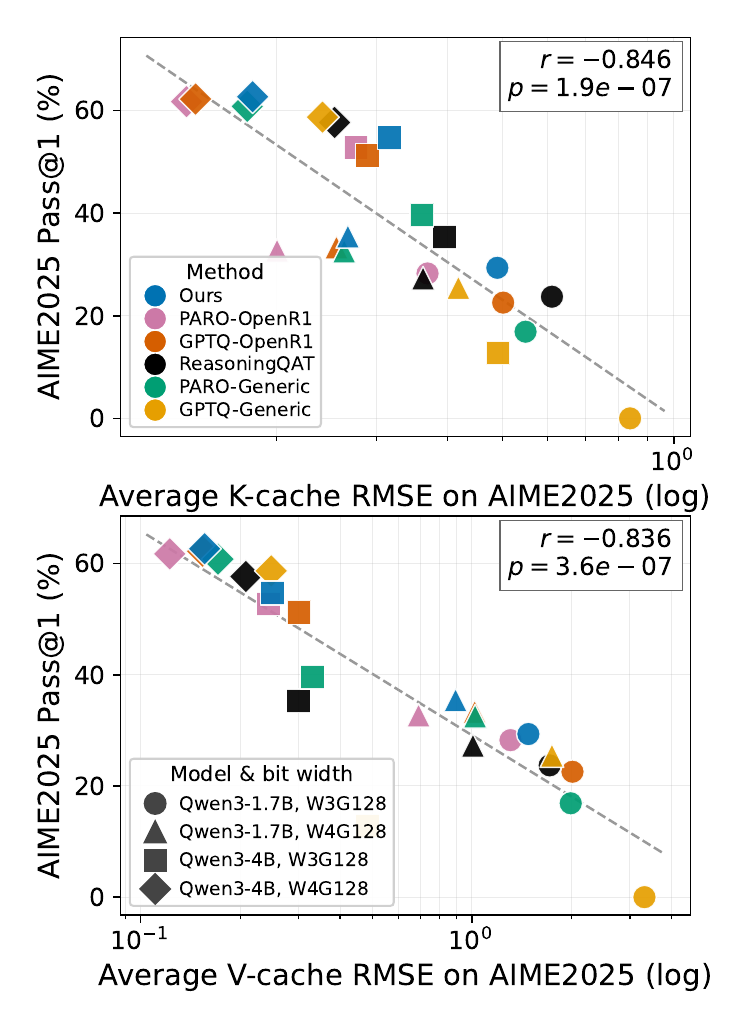}
    \caption{K-cache (top) and V-cache (bottom) RMSE vs.\ AIME25 Pass@1 across 24 quantized checkpoints.}
    \label{fig:kv_correlation}
  \end{subfigure}
  \caption{KV-cache fidelity is strongly associated with reasoning accuracy.
    (a) End-to-end QAT (ReasoningQAT) sacrifices both K- and V-cache fidelity across layers, while LAQuant preserves both. Top row: K-cache RMSE; bottom row: V-cache RMSE.
    (b) Across 24 quantized Qwen3-1.7B/4B checkpoints spanning 6 methods at W3G128 and W4G128, K- and V-cache RMSE on AIME25 traces both strongly anti-correlate with downstream AIME25 Pass@1; the two cache metrics are themselves highly correlated, consistent with a shared KV-cache fidelity factor.}
  \label{fig:motivation_kcache}
\end{figure}

\paragraph{KV-Cache Quality is Critical for Robust LRM Quantization.}
\label{sec:motivation1}
Recent studies~\cite{quantmeetsreasoning,quanthurtsreasoning,qwen3quant,quantlongcontext,reasoningqat,paroquant} report that legacy PTQ and QAT algorithms degrade on long-decoding reasoning tasks such as AIME~\cite{aime2024,aime2025} (competition math) and LiveCodeBench~\cite{lcb} (code generation), which generate up to 32k tokens per response~\cite{grpo,qwen3,deepseekr1,tts,s1}.
Weight quantization induces errors in the output activations of the key and value projections, which populate the KV cache attended by every subsequent token.
Because each generated token attends to all prior keys and values, both K-cache errors (corrupting the attention routing through $\mathrm{softmax}(QK^\top/\sqrt{d})$) and V-cache errors (corrupting the softmax-weighted output mixture) are consequential: they compound along two axes, namely (i) later transformer layers receive increasingly distorted inputs, and (ii) every future token must attend to a degraded KV-cache throughout decoding.

In contrast, conventional LLM quantization~\cite{gptq,awq,bitdistiller,efficientqat,omniquant,smoothquant,quarot,spinquant,rilq} is primarily evaluated on perplexity, zero-shot accuracy, and short-decode tasks, where KV-cache errors have little opportunity to compound along the sequence axis and output-level supervision (SFT loss or KL divergence against a BF16 teacher) proves highly effective; recent advances have accordingly moved toward end-to-end QAT~\cite{bitdistiller,efficientqat,spinquant,rilq}.

A natural hypothesis is that this logit-matching objective should carry over to long-decoding tasks. Our analysis (Figure~\ref{fig:motivation_kcache}a) refutes this: a recent end-to-end QAT model~\cite{reasoningqat} sacrifices both K- and V-cache fidelity across layers to match the BF16 teacher's logits under teacher-forcing, a trade that collapses under long autoregressive decoding where each generated token re-attends to the damaged KV cache. Our layer-wise model instead aligns intermediate hidden representations, achieving $1.25\times$ lower aggregate K-cache and $1.16\times$ lower aggregate V-cache RMSE than the E2E baseline, and recovers the reasoning accuracy quantified in Section~\ref{sec:reasoningresults}. The relationship is general: across 24 quantized checkpoints spanning 6 methods at W3G128 and W4G128, K- and V-cache RMSE (per-token average) on AIME25 traces both strongly anti-correlate with downstream Pass@1 (Pearson $r \approx -0.85$ for both), and are themselves highly correlated ($r = +0.85$), consistent with a shared KV-cache fidelity factor (Figure~\ref{fig:kv_correlation}).

\begin{figure}[!t]
  \centering
  \begin{subfigure}[b]{0.49\linewidth}
    \centering
    \includegraphics[width=\linewidth]{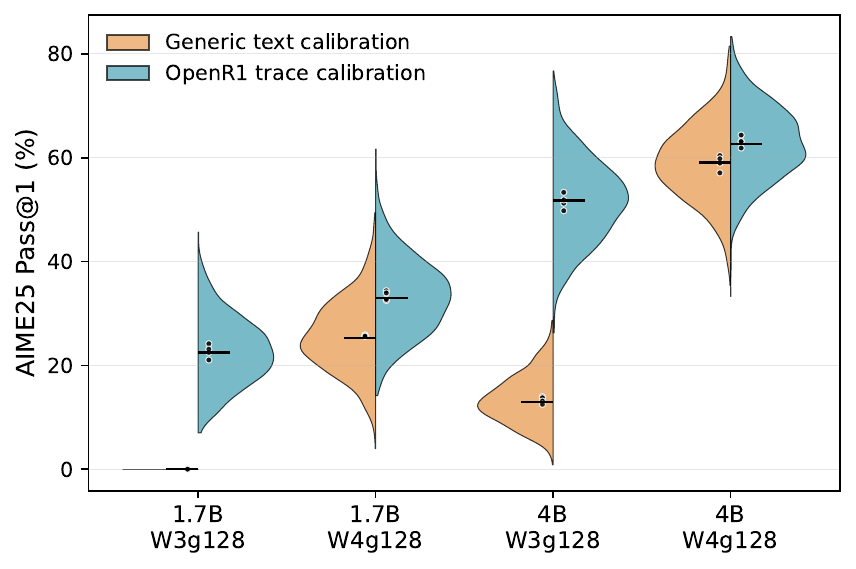}
    \caption{AIME25 Pass@1 with generic vs.\ reasoning calibration.}
    \label{fig:pass1_violin}
  \end{subfigure}\hfill
  \begin{subfigure}[b]{0.49\linewidth}
    \centering
    \includegraphics[width=\linewidth]{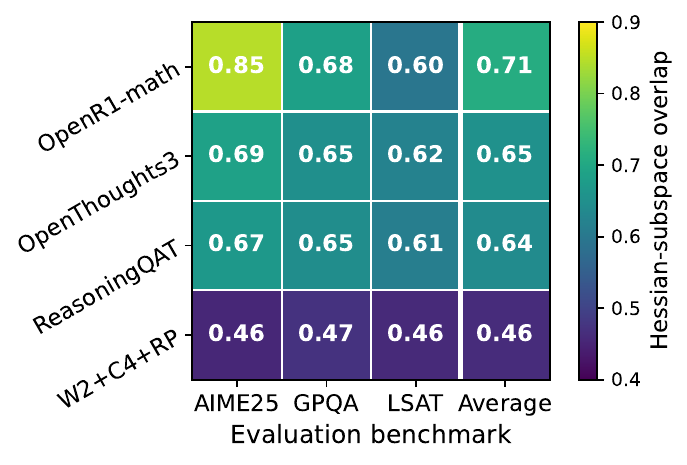}
    \caption{Hessian-subspace overlap (top-64 principal-angle cosine) between calibration corpora and evaluation traces on Qwen3-1.7B.}
    \label{fig:hessian} 
  \end{subfigure}
  \caption{Calibration data choice substantially affects reasoning accuracy.
    (a) Under the same GPTQ algorithm, switching the calibration corpus from a generic pretraining-text mix to DeepSeek-R1 reasoning traces yields large Pass@1 gains.
    (b) A plausible mechanism: reasoning-domain calibration corpora (OpenR1-math, OpenThoughts3, ReasoningQAT mix) induce a Hessian subspace that overlaps far more with the inference distribution than a generic-text mix (WikiText2+C4+RedPajama); the calibration ranking is invariant under $k \in [4, 256]$ (Appendix~\ref{app:topk_sensitivity}).}
\vspace{-1.0em}
  \label{fig:motivation_calib}
\end{figure}

\vspace{-1.0em}
\paragraph{LRM Quantization is Sensitive to Calibration Data, Especially under Challenging Scenarios.}
\label{sec:motivation2}
Existing LLM quantization algorithms typically calibrate on pre-training text corpora (e.g., WikiText2~\cite{wikitext2}, C4~\cite{c4}, RedPajama~\cite{redpajama}), and the generalization capability of modern LLMs has made this choice sufficient for most tasks.
For LRMs, however, calibration data choice is a first-order factor. Figure~\ref{fig:motivation_calib}a compares GPTQ-quantized LRMs under two calibration sources: a \textit{generic} mixture of pre-training text, and a \textit{reasoning} corpus of DeepSeek-R1 traces on OpenR1-Math-220k problems.
Reasoning calibration yields up to $38.45$\,pp improvement on AIME2025 Pass@1 (Qwen3-4B at W3G128 under GPTQ; Table~\ref{tab:res_main_reasoning_std}), with the largest gains at smaller model scales and lower bit widths --- e.g., on Qwen3-1.7B at W3G128, switching corpora moves GPTQ from $0.00$\% to $22.55$\% AIME25 Pass@1 (Appendix~\ref{app:full_reason_tasks}).
Prior works~\cite{reasoningqat,impactofcalib} have reported similar observations but only empirically; we provide a principled explanation based on Hessian-subspace alignment in Section \ref{sec:calibdata}.

\vspace{-1.0em}
\paragraph{Runtime-Overhead-Free Quantization Is Preferable.}
Rotation-based techniques~\cite{quarot,spinquant,paroquant,duquant} reduce per-layer quantization error, but every known formulation includes non-absorbable rotations applied online at every forward pass; even with efficient Hadamard kernels~\cite{fht,hadacore,qserve}, the added per-token latency dominates total cost in long-decoding LRM inference.
Online rotations also complicate deployment process: serving frameworks target general CUDA-core kernels and rarely include Hadamard primitives, and newer rotation variants~\cite{spinquant,duquant} require custom full-matrix transforms.
These considerations motivate LAQuant, which produces standard integer weights with per-group scales and zeros and deploys on any existing weight-only quantized kernel.

\section{LAQuant: Layer-Wise QAT with One-Layer Lookahead}
\label{sec:methodology}

We justify the design of LAQuant through two analytical strands, one for each factor identified in Section~\ref{sec:motivation}: (a) calibration-data alignment via activation Hessians (Section~\ref{sec:calibdata}), and (b) a gradient-direction analysis (Section~\ref{sec:grad_framework}) of supervision level (Section~\ref{sec:supervision}) and loss shape (Section~\ref{sec:lookahead}).

\subsection{Hessian-Subspace Analysis of Calibration Data}
\label{sec:calibdata}

A widespread empirical observation~\cite{quanthurtsreasoning, whatmakeslowbitquantizationawaretraining, reasoningqat, impactofcalib, thinkpruneselectiveselfgenerated} is that LRMs calibrated on reasoning traces substantially outperform those calibrated on generic text. We attribute this gap to an activation--Hessian subspace-alignment phenomenon, formalized as follows. For any calibration distribution \(\mathcal{D}_{\mathrm{cal}}\), let \(x\sim\mathcal{D}_{\mathrm{cal}}\) denote a calibration token sequence, and let \(\phi_\ell(x)\in\mathbb{R}^{d}\) denote the layer-\(\ell\) input activation induced by \(x\). The linearized layer-wise reconstruction objective with respect to the layer weight matrix \(W \in \mathbb{R}^{d_{\mathrm{out}} \times d}\) expands, near its unconstrained optimum \(W^\star\), as
\begin{equation}
  \mathcal{L}_{\mathrm{cal}}(W) \;\approx\; \mathrm{Tr}\!\left[(W - W^\star)\, H_{\mathrm{act}}^{\mathrm{cal}}\, (W - W^\star)^{\top}\right],
  \qquad
  H_{\mathrm{act}}^{\mathrm{cal}} \triangleq \mathbb{E}_{x \sim \mathcal{D}_{\mathrm{cal}}}\!\bigl[\phi_\ell(x)\phi_\ell(x)^{\top}\bigr],
  \label{eq:calib_quadratic}
\end{equation}
so the curvature is governed by the second-order activation statistics of the calibration distribution. Under a finite bit budget, the trained \(\hat W\) allocates precision to high-curvature eigendirections of \(H_{\mathrm{act}}^{\mathrm{cal}}\) and leaves low-curvature directions weakly constrained~\cite{obc, obd, obs, gptq}; if the deployment-activation distribution populates a different subspace, \(\hat W\) is systematically misallocated at inference.

The natural diagnostic for this mismatch is the average top-$k$ principal-angle cosine $\bar c_k(\mathcal{D}_{\mathrm{cal}}, \mathcal{D}_{\mathrm{eval}})$ between the top-$k$ eigenspaces of $H_{\mathrm{act}}^{\mathrm{cal}}$ and $H_{\mathrm{act}}^{\mathrm{eval}}$. LAQuant adopts a reasoning-trace calibration corpus (DeepSeek-R1 traces over OpenR1-Math-220k~\cite{openr1}); its alignment with deployment activations is reported in Figure~\ref{fig:hessian}.

\subsection{Gradient-Direction Framework}
\label{sec:grad_framework}

\begin{definition}[Per-layer quantization residual]
\label{def:residual}
For a quantized layer $Q_\ell$ with learnable parameters $\theta_q^{(\ell)}$, trained against teacher layer $F_\ell$, the per-layer quantization residual is
\begin{equation}
  \varepsilon_\ell(\theta_q^{(\ell)}) \;\triangleq\; Q_\ell\bigl(\hat h_\ell;\, \theta_q^{(\ell)}\bigr) - F_\ell(h^\star_\ell),
  \label{eq:residual}
\end{equation}
where $\hat h_\ell$ and $h^\star_\ell$ are the student-prefix and teacher-prefix layer-$\ell$ inputs respectively, with preceding student layers frozen and quantized.
\end{definition}

The residual is a high-dimensional vector in \(\mathbb{R}^{B \cdot T \cdot d_{\mathrm{hidden}}}\), whose downstream effect depends not only on its magnitude but also on its direction. Under a first-order linearization, components seen by sensitive downstream projections, such as the \(K\)-projection in layer \(\ell+1\), can perturb the attention keys and affect future-token attention routing. By contrast, components in directions weakly observed by the downstream Jacobian have limited first-order impact.

For any QAT loss $\mathcal{L}$ admitting a quadratic Taylor expansion in $\varepsilon_\ell$ near the unconstrained optimum, the gradient takes the form
\begin{equation}
  \nabla_{\theta_q^{(\ell)}}\, \mathcal{L} \;=\; J_\theta^{\top}\, M\, \varepsilon_\ell \;+\; O\bigl(\|\varepsilon_\ell\|^{2}\bigr),
  \qquad J_\theta \triangleq \partial \varepsilon_\ell / \partial \theta_q^{(\ell)},
  \label{eq:grad_form}
\end{equation}
where $M \succeq 0$ is a \emph{local} metric on the residual space, determined by the choice of loss together with the computation graph through which the residual is evaluated and the linearization point. The metric $M$ specifies how the residual's directional components are weighted on their way into the gradient; the next two subsections instantiate Eqn.~\eqref{eq:grad_form} for LAQuant's two loss-axis design decisions.

\subsection{Supervision Level: Layer-Wise over End-to-End}
\label{sec:supervision}

\begin{proposition}[Loss-induced metrics for E2E and layer-wise QAT; proof in Appendix~\ref{app:fisher_bound}]
\label{prop:metrics}
Under the local Gauss--Newton approximation around $\varepsilon_\ell = 0$, end-to-end QAT against the teacher's softmax with forward-KL loss yields
\begin{equation}
  M_{\mathrm{E2E}} \;=\; \mathcal{J}^{\top}\, H_{\mathrm{KL}}\, \mathcal{J},
  \qquad
  H_{\mathrm{KL}} \;=\; \mathrm{diag}(p^\star) - p^\star (p^\star)^{\top},
  \label{eq:M_E2E}
\end{equation}
where $\mathcal{J}$ maps the layer-$\ell$ residual to the pre-softmax logits through the remaining $L - \ell$ quantized layers, the final RMSNorm, and the language-model head, and $p^\star$ is the teacher's softmax distribution at the position. The layer-wise reconstruction loss (squared $L_2$ or its locally quadratic variants such as SmoothL1) yields, exactly,
\begin{equation}
  M_{\mathrm{lw}} \;=\; I.
  \label{eq:M_lw}
\end{equation}
\end{proposition}

The matrix $H_{\mathrm{KL}}$ is the Fisher information of the teacher's categorical distribution, with spectral norm bounded by $\sigma_{\max}(H_{\mathrm{KL}}) \leq 2 \max_i p^\star_i (1-p^\star_i) \leq 1/2$ (vanishing as $\max_i p^\star_i \to 1$; see Appendix~\ref{app:fisher_bound}); $M_{\mathrm{E2E}}$ is therefore the softmax Fisher metric carried through $\mathcal{J}$ to the residual.

Combining Proposition~\ref{prop:metrics} with this spectral bound gives $\|M_{\mathrm{E2E}}\, v\| \leq 2\max_i p^\star_i(1-p^\star_i) \cdot \sigma_{\max}(\mathcal{J})^{2} \cdot \|v\|$ for every residual direction $v$, whereas $\|M_{\mathrm{lw}}\, v\| = \|v\|$ (Proposition~\ref{prop:dilution}, Appendix~\ref{app:fisher_bound}). Beyond the operator-norm bound, $M_{\mathrm{E2E}}$ is strongly anisotropic in practice on AIME25 traces: residual perturbations along its leading eigendirection are substantially more responsive than along random directions (Appendix~\ref{app:e2e_anisotropy}). This anisotropy cannot be removed by any scalar learning-rate schedule or global loss reweighting since it is a structural property of the E2E gradient.

The KV-cache-reducing residual directions $g_K \triangleq \nabla_{h_\ell} \|\Delta K_{\ell+1}\|^{2}$ (and analogously $g_V$ for $V$), decisive for long-decoding reasoning (Section~\ref{sec:motivation1}), sit in the small-eigenvalue subspace of $M_{\mathrm{E2E}}$ rather than the leading subspace E2E supervision actually trains, with relative responsiveness comparable to a random-Gaussian baseline (Appendix~\ref{app:e2e_anisotropy}). LW supervision, with $M_{\mathrm{lw}} = I$, treats every residual direction equally and so retains the K/V learning signal. LAQuant therefore adopts pure layer-wise supervision (Appendix~\ref{app:peaky_softmax} confirms via a controlled E2E baseline); the lookahead loss in Section~\ref{sec:lookahead} contributes a complementary mechanism, cross-layer co-adaptation, analyzed next.

\subsection{Loss Shape: One-Layer Lookahead}
\label{sec:lookahead}

We refine the layer-wise objective with a one-layer lookahead (LA) loss
\begin{equation}
  \mathcal{L}_{\mathrm{la}}(\theta_q^{(\ell)}) \;\triangleq\; \mathrm{SmoothL1}\!\Bigl( F_{\ell+1}\!\bigl(Q_\ell(\hat h_\ell;\, \theta_q^{(\ell)})\bigr) ,\; F_{\ell+1}\!\bigl(F_\ell(h^\star_\ell)\bigr) \Bigr),
  \label{eq:lookahead_loss}
\end{equation}
which propagates the residual through the next teacher layer before measuring error. At the final layer $\ell = L$, $F_{\ell+1}$ does not exist; we set $\mathcal{L}_{\mathrm{la}}$ to the layer-wise reconstruction loss for that layer.

\begin{proposition}[Lookahead loss-induced metric (local first-order); proof in Appendix~\ref{app:fisher_bound}]
\label{prop:lookahead_metric}
Under a first-order expansion of $F_{\ell+1}$ around $F_\ell(h^\star_\ell)$, the lookahead loss~\eqref{eq:lookahead_loss} induces the local Gauss--Newton metric
\begin{equation}
  M_{\mathrm{la}} \;=\; J_{\ell+1}^{\top}\, J_{\ell+1},
  \qquad
  J_{\ell+1} \triangleq \partial F_{\ell+1} / \partial h \,\big|_{h = F_\ell(h^\star_\ell)}.
  \label{eq:M_LA}
\end{equation}
Unlike $M_{\mathrm{E2E}}$ from Proposition~\ref{prop:metrics}, $M_{\mathrm{la}}$ contains no vocabulary-softmax factor and no composition of multiple downstream Jacobians.
\end{proposition}

$M_{\mathrm{la}}$ inherits the \emph{Jacobian-weighted attribution} property of $M_{\mathrm{E2E}} = \mathcal{J}^{\top} H_{\mathrm{KL}} \mathcal{J}$: both reweight residual directions by the sensitivity of downstream computation to the residual. Unlike $M_{\mathrm{E2E}}$, however, $M_{\mathrm{la}}$ contains neither the softmax Fisher factor $H_{\mathrm{KL}}$ nor the composition of $L - \ell$ downstream Jacobians, both of which are responsible for E2E gradient dilution (Proposition~\ref{prop:dilution}). Lookahead supervision therefore retains the structural advantage of E2E (knowing where the residual will land in the next-layer computation) while avoiding the spectral attenuation that dilutes E2E gradients along the very residual directions that long-decoding reasoning needs.

Because $M_{\mathrm{la}}$ is shaped by layer $\ell{+}1$'s Jacobian rather than by layer $\ell$'s alone, the optimization of layer $\ell$ is informed by how layer $\ell{+}1$ will use the residual: the optimizer favors configurations whose projection errors \emph{combine} into a small next-layer residual rather than minimizing any single projection's marginal error. We call this property \emph{cross-layer co-adaptation}. The coupled projection set includes K and V, so lookahead reduces K- and V-projection output RMSE alongside the other five projections (Table~\ref{tab:ablation_per_projection_rmse}), directly improving the KV-cache fidelity that Section~\ref{sec:motivation1} identifies as critical for long-decoding accuracy. Table~\ref{tab:method_ablation_lookahead} further reports performance as the lookahead depth varies.

\section{Evaluation}
\label{sec:evaluation}

\paragraph{Models and Benchmarks.} \label{sec:model_benchmark}
We apply LAQuant on Llama-3.1 8B~\cite{llama3}, DeepSeek-R1-distilled Llama-3.1 8B~\cite{r1distillllama}, and Qwen3 (Base + chat variants at 1.7B, 4B, 8B)~\cite{qwen3}. Following ParoQuant~\cite{paroquant}, we evaluate LAQuant on three different types of benchmarks: i) perplexity on WikiText2 and C4, ii) zero-shot accuracy on ARC-Challenge, ARC-Easy~\cite{arc}, BoolQ~\cite{boolq}, and HellaSwag~\cite{hellaswag}, and iii) reasoning accuracy on AIME2024~\cite{aime2024}, AIME2025~\cite{aime2025}, GPQA Diamond~\cite{gpqa}, LSAT-AR~\cite{lsat}, MMLU-Pro~\cite{mmlupro}, and LiveCodeBench~\cite{lcb} (2024.10-2025.02 subset as used in the official Qwen3 evaluation\footnote{\url{https://qwen.ai/blog?id=qwen3}}).

\begin{table}[!ht]
\centering
\caption{Pass@1$^{\uparrow}$ on \textbf{reasoning tasks} under 3- and 4-bit quantization with a group size of 128. R-QAT and ParoQ denote ReasoningQAT and ParoQuant, respectively; ParoQ++ uses the same reasoning-domain dataset and training data size as LAQuant for a controlled experiment, and R-QAT uses its own data selection heuristic. Detailed results are reported in Appendix~\ref{app:full_reason_tasks}.}
\label{tab:res_main_reasoning}

\begin{subtable}[!ht]{\textwidth}
\centering
\caption{W3G128 Pass@1 on reasoning tasks.}
\label{tab:res_w3_reason}
\scriptsize
\setlength{\tabcolsep}{4pt}
\renewcommand{\arraystretch}{0.5}
\resizebox{\textwidth}{!}{%
\begin{tabular}{llcccccccc}
\toprule[1.2pt]
\textbf{Model} & \textbf{Method} & \textbf{AIME24} & \textbf{AIME25} & \textbf{GPQA} & \textbf{LSAT} & \textbf{MMLU-Pro} & \textbf{LCB} & \textbf{Avg.} & \textbf{Spd.} \\
\midrule
\multirow{6}{*}[-1ex]{Qwen3-1.7B}
& FP16   & 45.84 & 35.94 & 40.70 & 64.10 & 57.31 & 33.71 & 46.27 & 1.0$\times$ \\
\cmidrule(lr){2-10}
& GPTQ   & 0.00 & 0.00 & 25.03 & 18.78 & 22.34 & 0.00 & 11.03 & 2.29$\times$ \\
& R-QAT  & 25.52 & 23.70 & 31.03 & 47.69 & 47.48 & 19.90 & 32.55 & 2.29$\times$ \\
& ParoQ  & 7.24 & 16.88 & 30.49 & 36.10 & 43.79 & 12.94 & 24.57 & 2.05$\times$ \\
& ParoQ++  & 33.65 & 28.23 & 30.81 & 49.30 & 45.84 & 19.40 & 34.54 & 2.05$\times$\\
& LAQuant   & \textbf{34.64} & \textbf{29.33} & \textbf{31.79} & \textbf{54.86} & \textbf{47.94} & \textbf{21.02} & \textbf{36.60} & 2.29$\times$ \\
\midrule
\multirow{6}{*}[-1ex]{Qwen3-4B}
& FP16   & 73.23 & 63.75 & 55.59 & 81.14 & 70.23 & 54.48 & 66.40 & 1.0$\times$ \\
\cmidrule(lr){2-10}
& GPTQ   & 12.92 & 12.76 & 33.65 & 40.41 & 53.32 & 2.99 & 26.01 & 3.42$\times$ \\
& R-QAT  & 43.75 & 35.31 & \textbf{46.21} & 67.64 & \textbf{63.24} & 37.31 & 48.91 & 3.42$\times$ \\
& ParoQ  & 41.30 & 39.58 & 44.41 & 68.94 & 63.41 & 38.31 & 49.33 & 3.01$\times$ \\
& ParoQ++  & 59.74 & 52.76 & 46.17 & 77.07 & 62.41 & 45.03 & 57.20 & 3.01$\times$ \\
& LAQuant   & \textbf{62.92} & \textbf{54.69} & \textbf{46.21} & \textbf{77.15} & 62.63 & \textbf{45.15} & \textbf{58.12} & 3.42$\times$ \\
\midrule
\multirow{6}{*}[-1ex]{Qwen3-8B}
& FP16   & 74.95 & 67.66 & 58.78 & 85.57 & 73.44 & 58.46 & 69.81 & 1.0$\times$ \\
\cmidrule(lr){2-10}
& GPTQ   & 44.48 & 40.68 & 48.17 & 69.92 & 65.71 & 37.94 & 51.15 & 5.15$\times$ \\
& R-QAT  & 53.75 & 43.91 & 48.61 & 77.23 & 68.33 & 44.16 & 56.00 & 5.15$\times$ \\
& ParoQ  & 59.53 & 52.92 & 52.30 & 78.53 & \textbf{68.93} & 44.65 & 59.48 & 4.63$\times$ \\
& ParoQ++  & 67.19 & 58.39 & 53.09 & 79.48 & 63.32 & 51.12 & 62.10 & 4.63$\times$ \\
& LAQuant   & \textbf{68.33} & \textbf{61.30} & \textbf{53.66} & \textbf{83.80} & 66.86 & \textbf{51.74} & \textbf{64.28} & 5.15$\times$ \\
\midrule
\multirow{6}{*}[-1ex]{R1-Distill-Llama-8B}
& FP16   & 43.44 & 31.20 & 48.17 & 59.78 & 58.84 & 36.32 & 46.29 & 1.0$\times$ \\
\cmidrule(lr){2-10}
& GPTQ   & 15.45 & 19.64 & 37.97 & 42.47 & 47.10 & 24.13 & 31.13 & 5.56$\times$ \\
& R-QAT  & 30.16 & 27.50 & 40.50 & 48.48 & 51.05 & 29.11 & 37.80 & 5.56$\times$ \\
& ParoQ  & 0.89 & 0.05 & 32.99 & 29.65 & 44.29 & 1.12 & 18.17 & 4.90$\times$ \\
& ParoQ++  & 30.63 & 25.57 & 40.37 & 48.89 & 51.44 & 29.47 & 37.73 & 4.90$\times$ \\
& LAQuant   & \textbf{31.46} & \textbf{28.54} & \textbf{43.31} & \textbf{50.98} & \textbf{51.70} & \textbf{31.22} & \textbf{39.54} & 5.56$\times$ \\
\bottomrule[1.2pt]
\end{tabular}%
}
\end{subtable}

\vspace{0.5em}

\begin{subtable}[!ht]{\textwidth}
\centering
\caption{W4G128 Pass@1 on reasoning tasks.}
\label{tab:res_w4_reason}
\setlength{\tabcolsep}{4pt}
\renewcommand{\arraystretch}{0.5}
\resizebox{\textwidth}{!}{%
\begin{tabular}{llcccccccc}
\toprule[1.2pt]
\textbf{Model} & \textbf{Method} & \textbf{AIME24} & \textbf{AIME25} & \textbf{GPQA} & \textbf{LSAT} & \textbf{MMLU-Pro} & \textbf{LCB} & \textbf{Avg.} & \textbf{Spd.} \\
\midrule
\multirow{6}{*}[-1ex]{Qwen3-1.7B}
& FP16   & 45.84 & 35.94 & 40.70 & 64.10 & 57.31 & 33.71 & 46.27 & 1.0$\times$ \\
\cmidrule(lr){2-10}
& GPTQ   & 24.17 & 25.47 & 33.14 & 48.29 & 50.49 & 20.40 & 33.66 & 1.64$\times$ \\
& R-QAT  & 32.61 & 27.29 & 35.70 & 53.48 & 54.44 & 25.87 & 38.23 & 1.64$\times$ \\
& ParoQ  & 39.53 & 32.55 & 35.86 & 58.30 & 54.58 & 28.61 & 41.57 & 1.48$\times$ \\
& ParoQ++  & 41.82 & 32.71 & 36.30 & 60.52 & 55.44 & 28.98 & 42.63 & 1.48$\times$ \\
& LAQuant   & \textbf{42.35} & \textbf{35.47} & \textbf{36.43} & \textbf{61.14} & \textbf{55.58} & \textbf{31.84} & \textbf{43.80} & 1.64$\times$ \\
\midrule
\multirow{6}{*}[-1ex]{Qwen3-4B}
& FP16   & 73.23 & 63.75 & 55.59 & 81.14 & 70.23 & 54.48 & 66.40 & 1.0$\times$ \\
\cmidrule(lr){2-10}
& GPTQ   & 63.44 & 58.67 & 51.61 & 78.10 & 68.03 & 48.75 & 61.43 & 2.0$\times$ \\
& R-QAT  & 64.64 & 57.66 & 52.15 & 78.45 & 68.28 & 46.76 & 61.32 & 2.0$\times$ \\
& ParoQ  & 67.29 & 60.78 & 52.46 & 79.43 & \textbf{69.41} & 50.00 & 63.23 & 1.84$\times$ \\
& ParoQ++  & 70.37 & 61.72 & 53.44 & 80.87 & 68.73 & 51.62 & 64.46 & 1.84$\times$ \\
& LAQuant   & \textbf{70.42} & \textbf{62.66} & \textbf{53.60} & \textbf{81.03} & 68.41 & \textbf{52.74} & \textbf{64.81} & 2.0$\times$ \\
\midrule
\multirow{6}{*}[-1ex]{Qwen3-8B}
& FP16   & 74.95 & 67.66 & 58.78 & 85.57 & 73.44 & 58.46 & 69.81 & 1.0$\times$ \\
\cmidrule(lr){2-10}
& GPTQ   & 71.56 & 65.21 & 57.89 & 82.31 & \textbf{72.67} & 52.99 & 67.10 & 2.45$\times$ \\
& R-QAT  & 68.44 & 59.48 & \textbf{58.49} & \textbf{85.24} & 72.65 & 51.62 & 65.99 & 2.45$\times$ \\
& ParoQ  & 73.39 & 64.52 & 56.85 & 84.43 & 72.43 & 55.47 & 67.85 & 2.32$\times$ \\
& ParoQ++  & 73.81 & 64.27 & 56.66 & 83.99 & 72.59 & 55.72 & 67.84 & 2.32$\times$ \\
& LAQuant   & \textbf{74.17} & \textbf{66.98} & 57.23 & \textbf{85.24} & 71.52 & \textbf{56.69} & \textbf{68.62} & 2.45$\times$ \\
\midrule
\multirow{6}{*}[-1ex]{R1-Distill-Llama-8B}
& FP16   & 43.44 & 31.20 & 48.17 & 59.78 & 58.84 & 36.32 & 46.29 & 1.0$\times$ \\
\cmidrule(lr){2-10}
& GPTQ   & 38.44 & 29.85 & \textbf{46.65} & 56.66 & \textbf{56.78} & 34.21 & 43.66 & 2.46$\times$ \\
& R-QAT  & 40.63 & 29.53 & 46.09 & 57.01 & 55.96 & 34.33 & 43.92 & 2.46$\times$ \\
& ParoQ  & 39.12 & 28.54 & 46.34 & 54.81 & 56.69 & 33.33 & 43.14 & 2.33$\times$ \\
& ParoQ++  & 38.86 & 27.60 & 46.46 & 56.66 & 56.77 & 35.45 & 43.63 & 2.33$\times$ \\
& LAQuant   & \textbf{41.41} & \textbf{31.77} & 45.55 & \textbf{58.67} & 54.98 & \textbf{35.70} & \textbf{44.68} & 2.46$\times$ \\
\bottomrule[1.2pt]
\end{tabular}%
}
\end{subtable}

\end{table}

\paragraph{Implementation Details.}
We use group-128 weight-only quantization at 3 and 4 bits. For reasoning evaluation, QAT data is the first 4,096 samples of OpenR1-Math-220k~\cite{openr1} (verified disjoint from our benchmarks); we apply the Qwen-style chat template to prompts and DeepSeek-R1 responses and right-truncate to 2,048 tokens including formatting. For fair comparison, GPTQ is calibrated using the same data budget as the QAT-based methods in each setting. For base models on general tasks, we use a WikiText2/C4/RedPajama mixture.

Our framework builds on the official EfficientQAT~\cite{efficientqat} implementation, training each layer for 20 epochs with embedding and LM head frozen. Weights, quantization parameters (scales and zeros), learnable clipping bounds~\cite{omniquant}, and per-channel scaling~\cite{smoothquant} are jointly trained with AdamW~\cite{adamw}. Following ParoQuant~\cite{paroquant}, we use SmoothL1~\cite{fastrcnn,huber} as the reconstruction loss. The learning rate for weights is set to half that of the other trainable quantities ($2.5{\times}10^{-5}$/$5{\times}10^{-5}$ for Llama, $5{\times}10^{-5}$/$1{\times}10^{-4}$ for Qwen3).

\subsection{Accuracy Results}

\paragraph{Reasoning Tasks.}
\label{sec:reasoningresults}
Table~\ref{tab:res_main_reasoning} reports Pass@1 on six reasoning tasks under 4-bit and 3-bit quantization. The results first show that reasoning-domain calibration data are crucial for preserving reasoning ability in the low-bit regime. Under W3G128, ParoQuant++ consistently outperforms ParoQuant across the models, with the average accuracy on Qwen3-1.7B improving from 24.57 to 34.54, corresponding to a \textbf{40.6\%} relative gain. The effect is most pronounced on R1-Distill-Llama-8B at W3G128, where ParoQuant's default calibration collapses AIME25 Pass@1 to $0.05$ while ParoQuant++ recovers it to $25.57$ --- a $\sim$20\,pp six-task average gain ($18.17\to 37.73$) attributable to calibration choice alone.

The results also highlight the advantage of lookahead quantization as an efficient alternative to online transformations. Compared with ParoQuant++, our method further improves reasoning performance while eliminating inference overhead. On DeepSeek-R1-Distill-Llama-8B, LAQuant improves the average accuracy from 43.63 to 44.68, corresponding to a \textbf{2.4\%} relative gain, while also increasing decoding speedup from 2.33$\times$ to 2.46$\times$. This shows that lookahead quantization can provide additional reasoning performance without incurring the inference overhead of online transformations. Detailed results for each reasoning benchmark are provided in Appendix~\ref{app:full_reason_tasks}.
\begin{table}[!ht]
\centering
\caption{Results on \textbf{general tasks} under 4- and 3-bit quantization. Perplexity is evaluated at a context length of 8192. 0-shot denotes the average accuracy across ARC-C, ARC-E, BoolQ, and HellaSwag.}
\label{tab:res_main_general}
\scriptsize
\setlength{\tabcolsep}{2pt}
\renewcommand{\arraystretch}{0.7}

\begin{tabular}{llccc ccc ccc ccc}
\toprule[1.2pt]
\multirow{2}{*}{Method} & \multirow{2}{*}{Bits}
& \multicolumn{3}{c}{\textbf{Qwen3-1.7B-Base}}
& \multicolumn{3}{c}{\textbf{Qwen3-4B-Base}}
& \multicolumn{3}{c}{\textbf{Qwen3-8B-Base}}
& \multicolumn{3}{c}{\textbf{Llama3.1-8B}} \\
\cmidrule(lr){3-5} \cmidrule(lr){6-8} \cmidrule(lr){9-11} \cmidrule(lr){12-14}
&
& WikiText2$^{\downarrow}$ & C4$^{\downarrow}$ & 0-shot$^{\uparrow}$
& WikiText2$^{\downarrow}$ & C4$^{\downarrow}$ & 0-shot$^{\uparrow}$
& WikiText2$^{\downarrow}$ & C4$^{\downarrow}$ & 0-shot$^{\uparrow}$
& WikiText2$^{\downarrow}$ & C4$^{\downarrow}$ & 0-shot$^{\uparrow}$ \\
\midrule
FP16
& 16
& 8.31 & 8.71 & 59.95
& 7.01 & 7.63 & 66.28
& 6.23 & 6.97 & 68.95
& 5.61 & 6.76 & 68.68 \\
\cmidrule(lr){3-14}
GPTQ
& 3
& 10.56 & 10.88 & 51.72
& 8.12 & 8.71 & 60.11
& 6.96 & 7.65 & 67.34
& 6.79 & 7.90 & 63.08 \\
ParoQ
& 3
& \textbf{8.91} & \textbf{9.38} & 58.53
& 7.57 & 8.21 & \textbf{66.01}
& \textbf{6.56} & \textbf{7.32} & 68.61
& 6.38 & 7.49 & 65.51 \\
LAQuant
& 3
& 8.99 & 9.52 & \textbf{58.61}
& \textbf{7.44} & \textbf{8.16} & 65.39
& 6.57 & 7.37 & \textbf{69.35}
& \textbf{6.31} & \textbf{7.48} & \textbf{65.64} \\
\cmidrule(lr){3-14}
GPTQ
& 4
& 8.74 & 9.14 & 58.95
& 7.23 & 7.85 & 65.79
& 6.39 & 7.11 & \textbf{69.48}
& 5.85 & 6.98 & 68.25 \\
ParoQ
& 4
& \textbf{8.43} & \textbf{8.83} & \textbf{60.75}
& \textbf{7.09} & \textbf{7.73} & \textbf{67.36}
& \textbf{6.29} & \textbf{7.04} & 69.32
& \textbf{5.78} & \textbf{6.91} & 68.12 \\
LAQuant
& 4
& 8.46 & 8.91 & 60.20
& 7.12 & 7.77 & 66.44
& 6.32 & 7.07 & 69.32
& 5.80 & 6.94 & \textbf{68.36} \\
\bottomrule[1.2pt]
\end{tabular}
\end{table}
\vspace{-1.0em}
\paragraph{General Tasks.}
On general tasks (Table~\ref{tab:res_main_general}), LAQuant matches ParoQuant on perplexity and zero-shot accuracy across all evaluated models in both W4G128 and W3G128, while avoiding online transformations. Per-task results are provided in Appendix~\ref{app:full_general_tasks}.
\subsection{Efficiency Results}
\vspace{-2.0em}
\paragraph{Inference.}
\begin{table}[!ht]
\centering
\caption{Decoding (batch size = 1) throughput (tokens/sec)$^{\uparrow}$ on RTX A6000.}
\label{tab:rtx_a6000_sm86}
\scriptsize
\setlength{\tabcolsep}{3pt}
\renewcommand{\arraystretch}{0.75}
\begin{tabular}{lccccc}
\toprule[1.2pt]
RTX A6000 & Bits & Qwen3-1.7B & Qwen3-4B & Qwen3-8B & Llama3.1-8B \\
\midrule
FP16    & 16 & 138.4         & 67.6          & 38.2     & 40.47     \\
\midrule
ParoQuant & 4  & 205.3 (x1.48) & 124.5 (x1.84) & 88.7 (x2.32) &  94.35 (x2.33) \\
LAQuant      & 4  & \textbf{227.3} (x1.64) & \textbf{135.1} (x2.00) & \textbf{93.8} (x2.45) & \textbf{99.74} (x2.46) \\
\midrule
ParoQuant & 3  & 283.7 (x2.05) & 203.5 (x3.01) & 177.0 (x4.63) & 198.13 (x4.90) \\
LAQuant      & 3  & \textbf{316.9} (x2.29) & \textbf{231.3} (x3.42) & \textbf{196.8} (x5.15) & \textbf{225.01} (x5.56) \\
\bottomrule[1.2pt]
\end{tabular}
\end{table}
Table~\ref{tab:rtx_a6000_sm86} shows decoding throughput on an RTX A6000. For INT3, we adopt the GPTQ GEMV kernel with additional tuning for smaller model sizes, since the original kernel was designed for 175B-scale models; we describe the modifications in detail in Appendix~\ref{app:int3_kernel}. To ensure a fair comparison, we integrate our INT3 kernel into the ParoQuant benchmark pipeline (\texttt{StaticCache} + \texttt{torch.compile} + CUDA graphs) and run both with and without the online rotation. For INT4, both LAQuant and ParoQuant use the same AWQ GEMV kernel; the only difference is whether the rotation is applied.

Removing the online rotation consistently improves throughput: 6--11\% over ParoQuant at W4 (avg 7.7\%) and 11--14\% at W3 (avg 12.5\%). The W3 gain is larger because the GEMV is cheaper, so the fixed rotation overhead is a larger share of per-layer latency; sequential decoding then accumulates this over long generations. Results on more GPUs are provided in Appendix~\ref{app:full_decode_throughput}.

\paragraph{Training.}
\label{sec:training}
\begin{table}[!ht]
\centering
\caption{Comparison of resource usage across ParoQuant, ReasoningQAT, and LAQuant.}
\label{tab:res_cost}
\resizebox{\linewidth}{!}{%
\begin{tabular}{lccccccccccc}
\toprule[1.2pt]
& & & \multicolumn{3}{c}{ParoQuant} & \multicolumn{3}{c}{ReasoningQAT} & \multicolumn{3}{c}{LAQuant} \\
\cmidrule(lr){4-6} \cmidrule(lr){7-9} \cmidrule(lr){10-12}
Model & Batch & Shard. & Train tokens & VRAM (GB) & Time (h) & Train tokens & VRAM (GB) & Time (h) & Train tokens & VRAM (GB) & Time (h) \\
\midrule
Q3-1.7B & 32 &  & 167.8M & 113.9 & 3.29  & 285.2M & 79.79 & 30.96 & 167.8M & 102.2 & 1.82  \\
Q3-4B   & 16 &  & 167.8M & 138.7 & 8.76  & 285.2M & 98.03 & 66.41 & 167.8M & 126.5 & 4.63  \\
Q3-8B   & 16 & \checkmark & 167.8M & 121.3 & 19.82 & 285.2M & 126.3 & 78.44 & 167.8M & 119.2 & 14.91 \\
\bottomrule[1.2pt]
\end{tabular}%
}
\end{table}
Table~\ref{tab:res_cost} compares training efficiency across ParoQuant, ReasoningQAT, and LAQuant on a single H200 144GB GPU. For a fair comparison with ParoQuant at matched accuracy, both methods use the same data budget (4,096 samples, length 2,048, 20 epochs); ReasoningQAT already uses 1.7$\times$ more tokens under its original two-stage training pipeline, so we keep both stages in its prescribed configuration. Each method runs at its largest feasible batch size, with \emph{Shard.}\ denoting CPU offloading. LAQuant trains faster across all model sizes thanks to its lightweight optimization (no online-transform parameters, fewer tokens than ReasoningQAT); on Qwen3-8B, training drops from 19.82 to 14.91 GPU hours (24.8\% reduction) at slightly lower peak VRAM.

\subsection{Ablation Study}
We present two complementary ablations in Table~\ref{tab:method_ablation}: a task-level analysis of lookahead depth and a projection-level output RMSE analysis with and without lookahead. The experimental settings and additional ablations are provided in Appendix~\ref{app:additional_abl_study}.
\begin{table}[!ht]
\centering
\caption{Method ablations on Qwen3-1.7B (W3G128); (b) also evaluated on Qwen3-4B.}
\label{tab:method_ablation}

\begin{subtable}[!ht]{0.32\textwidth}
\centering
\caption{Lookahead depth $k$: Pass@1 on AIME25 / LSAT, with their average.}
\label{tab:method_ablation_lookahead}
\scriptsize
\setlength{\tabcolsep}{3pt}
\renewcommand{\arraystretch}{0.95}
\begin{tabular}{lcccc}
\toprule
 & $k{=}0$ & $k{=}0{\to}1$ & $k{=}1$ & $k{=}2$ \\
\midrule
AIME25 & 26.88 & 28.39 & 29.33 & 28.44 \\
LSAT   & 52.15 & 52.36 & 54.86 & 53.94 \\
Avg.   & 39.52 & 40.38 & \textbf{42.10} & 41.19 \\
\bottomrule
\end{tabular}
\end{subtable}
\hfill
\begin{subtable}[!ht]{0.63\textwidth}
\centering
\caption{Per-projection output RMSE on AIME25 traces, with ($k{=}1$) vs without ($k{=}0$) lookahead.}
\label{tab:ablation_per_projection_rmse}
\scriptsize
\renewcommand{\arraystretch}{0.7}
\setlength{\tabcolsep}{1.5pt}
\begin{tabular}{llrrrrrrr}
\toprule[1.2pt]
\textbf{Model}
& \textbf{$k$}
& \textbf{Q} & \textbf{K} & \textbf{V} & \textbf{O}
& \textbf{Gate} & \textbf{Up} & \textbf{Down} \\
\midrule
\multirow{3}{*}{1.7B}
& $0$    & 1.8729 & 0.4935 & 1.5039 & 0.8288 & 0.3675 & 0.3678 & 1.2166 \\
& $1$    & 1.8442 & 0.4892 & 1.4812 & 0.8089 & 0.3590 & 0.3615 & 1.2025 \\
& $\Delta$ & $-1.53\%$ & $-0.88\%$ & $-1.51\%$ & $-2.40\%$ & $-2.31\%$ & $-1.70\%$ & $-1.16\%$ \\
\midrule
\multirow{3}{*}{4B}
& $0$    & 0.2821 & 0.3188 & 0.2545 & 0.1344 & 0.1919 & 0.1756 & 0.2502 \\
& $1$    & 0.2783 & 0.3162 & 0.2490 & 0.1304 & 0.1893 & 0.1724 & 0.2411 \\
& $\Delta$ & $-1.34\%$ & $-0.81\%$ & $-1.79\%$ & $-2.95\%$ & $-1.41\%$ & $-1.81\%$ & $-3.64\%$ \\
\bottomrule[1.2pt]
\end{tabular}
\end{subtable}
\end{table}
\vspace{-1.0em}

\paragraph{Lookahead Depth.}
Table~\ref{tab:method_ablation_lookahead} reports lookahead depth $k$ (Section~\ref{sec:lookahead}), with $k{=}0$ the pure layer-wise objective, $k{\geq}1$ a $k$-layer lookahead, and $k{=}0{\to}1$ a staged schedule. Lifting $k$ from 0 to 1 improves the AIME25 / LSAT average from 39.52 to 42.10; the staged schedule reaches only 40.38, and $k{=}2$ adds training cost without accuracy gain. We default to $k=1$.

\paragraph{Per-Projection Output RMSE.}
With the lookahead loss, all seven projections improve by $0.8$--$3.6\%$ (Table~\ref{tab:ablation_per_projection_rmse}). The largest reductions land on the attention and MLP output projections (o\_proj $-2.95\%$, down\_proj $-3.64\%$ at Qwen3-4B), which contribute directly to the next-layer input and so receive the strongest reweighting under cross-layer co-adaptation (Section~\ref{sec:lookahead}).

\section{Conclusion}

We presented LAQuant, a single-stage layer-wise weight-only QAT method for large reasoning models. Our gradient-direction analysis identifies two factors driving the LRM-quantization difficulty gap: end-to-end supervision attenuates per-layer learning signals (including for KV-cache directions) via the softmax Fisher metric, and calibration-data choice is governed by Hessian-subspace alignment with the deployment distribution. LAQuant addresses both with a one-layer lookahead loss whose implicit cross-layer co-adaptation preserves the next-layer residual stream, paired with a reasoning-domain calibration corpus, producing standard integer weights deployable on any existing weight-only kernel. On Qwen3 and DeepSeek-R1-distilled Llama models, LAQuant attains the highest average reasoning accuracy among weight-only methods at W3G128 and W4G128 while delivering a 3.42$\times$ decoding speedup over FP16 on RTX A6000.

\newpage

\bibliographystyle{plain}
\bibliography{ref}

\newpage


\appendix

\section{Related Work}
\label{app:related_work}

\paragraph{Post-training quantization (PTQ) for LLMs.}
GPTQ~\cite{gptq} introduced layer-wise OBS-style weight reconstruction with per-group quantization, becoming the dominant PTQ recipe for $\geq 4$-bit weight-only quantization. AWQ~\cite{awq} adds per-channel scaling absorbed into the preceding LayerNorm to mitigate activation outliers. SqueezeLLM~\cite{squeezellm} introduces sensitivity-based non-uniform quantization paired with a dense-and-sparse decomposition for outlier handling. SmoothQuant~\cite{smoothquant} and OmniQuant~\cite{omniquant} extend the channel-scaling and clipping ideas with learnable parameters. Rotation-based methods including QuaRot~\cite{quarot}, SpinQuant~\cite{spinquant}, and DuQuant~\cite{duquant} insert orthogonal transforms (typically Hadamard) to mix outlier directions into the bulk before quantization, achieving 4-bit all-quantization but introducing online rotation overhead at inference. Variable-bit-per-layer schemes~\cite{varlayerquant} assign different bit widths across layers based on layer-importance signals, complementing fixed-bit recipes. Lower-bit weight-only methods reach 2-bit~\cite{efficientqat,bitdistiller,dbllm,vptq} or even 1-bit~\cite{billm,bitnet} regimes.

\paragraph{Quantization-aware training (QAT) for LLMs.}
EfficientQAT~\cite{efficientqat} introduced block-wise QAT with learnable quantization parameters followed by an optional end-to-end (E2E) phase, demonstrating that scales and zeros can be trained jointly with weights at scale. BitDistiller~\cite{bitdistiller} and RILQ~\cite{rilq} apply teacher--student distillation in an E2E fashion. SpinQuant~\cite{spinquant} performs an additional E2E phase over rotation parameters.

\paragraph{Quantization for large reasoning models.}
ReasoningQAT~\cite{reasoningqat} chains layer-wise and E2E QAT with a domain-mixed calibration corpus, the first quantization pipeline targeting LRMs specifically. ParoQuant~\cite{paroquant} combines minimal pairwise rotation with channel scaling to reduce per-layer error. QuantLRM~\cite{quantlrm} leverages weight-update statistics from reasoning-incentivized fine-tuning to identify per-channel importance for LRM quantization. Concurrent observations~\cite{quantmeetsreasoning,quanthurtsreasoning,qwen3quant,quantlongcontext} report that legacy quantization recipes degrade markedly on long-decoding reasoning benchmarks; recent analyses~\cite{quantrlrm} further study the interaction between quantization and reasoning-incentivized RL training.

\paragraph{Calibration data selection.}
Pre-training corpora (WikiText2~\cite{wikitext2}, C4~\cite{c4}, RedPajama~\cite{redpajama}) have been the default calibration source. Recent work~\cite{impactofcalib} reports that domain-aligned calibration improves task accuracy, with effects largest at low bit widths and on instruction-tuned models. We extend this empirical observation with a Hessian-subspace-alignment formalization in Section~\ref{sec:calibdata}.

\section{Proofs of Propositions 1--3 and Spectral Bound}
\label{app:fisher_bound}

We restate and prove the spectral-norm bound on $H_{\mathrm{KL}}$ used in Section~\ref{sec:supervision}.

\begin{lemma}[Spectral-norm bound on the categorical Fisher]
\label{lem:fisher_bound}
For any $p^\star \in \Delta^{V-1}$,
\begin{equation}
  \sigma_{\max}\bigl(\mathrm{diag}(p^\star) - p^\star (p^\star)^{\top}\bigr) \;\leq\; 2 \max_i p^\star_i (1 - p^\star_i) \;\leq\; \tfrac{1}{2},
  \label{eq:fisher_bound}
\end{equation}
with the upper bound vanishing as $\max_i p^\star_i \to 1$.
\end{lemma}

\begin{proof}
Let $A = \mathrm{diag}(p^\star) - p^\star (p^\star)^{\top}$. The diagonal entries are $A_{ii} = p^\star_i(1-p^\star_i)$ and the off-diagonals are $A_{ij} = -p^\star_i p^\star_j$. Since $A \succeq 0$, $\sigma_{\max}(A)$ equals its largest eigenvalue. By Gershgorin's circle theorem,
\[
\sigma_{\max}(A) \;\leq\; \max_i \Bigl( A_{ii} + \sum_{j \neq i} |A_{ij}| \Bigr) \;=\; \max_i \Bigl( p^\star_i(1-p^\star_i) + p^\star_i \!\!\sum_{j \neq i} p^\star_j \Bigr) \;=\; 2 \max_i p^\star_i (1-p^\star_i),
\]
using $\sum_{j \neq i} p^\star_j = 1 - p^\star_i$. The map $t \mapsto t(1-t)$ on $[0,1]$ attains its maximum $1/4$ at $t = 1/2$, giving $2 \max_i p^\star_i(1-p^\star_i) \leq 1/2$. As $\max_i p^\star_i \to 1$, all probabilities approach $0$ or $1$, so $\max_i p^\star_i(1-p^\star_i) \to 0$ and the bound vanishes.
\end{proof}

\begin{proposition}[E2E gradient attenuation]
\label{prop:dilution}
Combining Proposition~\ref{prop:metrics} with Lemma~\ref{lem:fisher_bound}, every direction $v$ in the per-layer residual space satisfies
\begin{equation}
  \|M_{\mathrm{E2E}}\, v\| \;\leq\; 2 \max_i p^\star_i (1 - p^\star_i) \;\cdot\; \sigma_{\max}(\mathcal{J})^{2} \;\cdot\; \|v\|,
  \label{eq:e2e_dilution}
\end{equation}
whereas $\|M_{\mathrm{lw}}\, v\| = \|v\|$. Inequality~\eqref{eq:e2e_dilution} bounds the operator norm of $M_{\mathrm{E2E}}$ by the product of the categorical-Fisher spectral norm and the squared spectral norm of the downstream Jacobian: when $p^\star$ is sharply peaked and $\mathcal{J}$ is non-expansive, both factors are small and E2E gradients can be strongly attenuated, with the second factor decaying geometrically only when the downstream Jacobians are strictly contractive on average across depth.
\end{proposition}

\begin{proof}[Proof of Proposition~\ref{prop:metrics}]
The layer-wise reconstruction loss $\mathcal{L}_{\mathrm{lw}} = \tfrac{1}{2}\|\varepsilon_\ell\|^2$ has exact Hessian $I$ in $\varepsilon_\ell$, so $M_{\mathrm{lw}} = I$. For E2E, the forward-KL loss against the teacher's softmax distribution $p^\star$ (equivalently, cross-entropy with the fixed target $p^\star$) has Hessian $H_{\mathrm{KL}} = \mathrm{diag}(p^\star) - p^\star p^{\star\top}$ in the pre-softmax logits $z$ (the categorical Fisher information of $p^\star$). Composing with $\mathcal{J} = \partial z / \partial \varepsilon_\ell$ via the chain rule, the Hessian of the E2E loss in $\varepsilon_\ell$ is $\mathcal{J}^{\top} H_{\mathrm{KL}} \mathcal{J} + \sum_i (p_i - p^\star_i)\, \nabla^2_{\varepsilon_\ell} z_i$. The local Gauss--Newton approximation drops the residual second-derivative term --- which vanishes exactly when the student matches the teacher downstream and is small in a neighborhood of that point --- giving $M_{\mathrm{E2E}} = \mathcal{J}^{\top} H_{\mathrm{KL}} \mathcal{J}$.
\end{proof}

\begin{proof}[Proof of Proposition~\ref{prop:lookahead_metric}]
Taylor-expanding $F_{\ell+1}$ around $F_\ell(h^\star_\ell)$ gives
$F_{\ell+1}(F_\ell(h^\star_\ell) + \varepsilon_\ell) - F_{\ell+1}(F_\ell(h^\star_\ell)) = J_{\ell+1}\, \varepsilon_\ell + O(\|\varepsilon_\ell\|^2)$,
where $\varepsilon_\ell = Q_\ell(\hat h_\ell;\theta_q^{(\ell)}) - F_\ell(h^\star_\ell)$ and $J_{\ell+1} = \partial F_{\ell+1}/\partial h$ at $h = F_\ell(h^\star_\ell)$. Substituting into $\mathcal{L}_{\mathrm{la}}$, which reduces to $\tfrac{1}{2}\|\cdot\|^2$ below the Huber threshold (Section~\ref{sec:lookahead}), yields the local quadratic form $\tfrac{1}{2}\,\varepsilon_\ell^{\top}\, J_{\ell+1}^{\top} J_{\ell+1}\, \varepsilon_\ell + O(\|\varepsilon_\ell\|^3)$, hence $M_{\mathrm{la}} = J_{\ell+1}^{\top} J_{\ell+1}$.
\end{proof}

\begin{proof}[Proof of Proposition~\ref{prop:dilution}]
By Proposition~\ref{prop:metrics} and sub-multiplicativity of operator norms,
$\|M_{\mathrm{E2E}}\, v\| \leq \sigma_{\max}(\mathcal{J})^{2}\, \sigma_{\max}(H_{\mathrm{KL}})\, \|v\|$.
Applying Lemma~\ref{lem:fisher_bound} to bound $\sigma_{\max}(H_{\mathrm{KL}})$ gives the stated inequality. The LW case is immediate from $M_{\mathrm{lw}} = I$.
\end{proof}

\section{Direction-Stratified Anisotropy and Top-1 Agreement}
\label{app:e2e_anisotropy}

We empirically distinguish two scenarios consistent with the operator-norm bound of Section~\ref{sec:supervision}: \emph{(i) uniform shrinkage} (compensable by a larger learning rate), or \emph{(ii) anisotropic attenuation} where the leading subspace of $M_{\mathrm{E2E}}$ is small and the rest --- including KV-cache-reducing directions --- decays sharply outside it (no scalar reweighting can repair it). Direct measurement on AIME25 traces confirms (ii). A token-weighted top-1 agreement comparison against the BF16 teacher provides the corresponding downstream confirmation.

\paragraph{Setup.}
We compare the LW ($k{=}0$) student of Table~\ref{tab:method_ablation_lookahead} and a controlled-E2E baseline (Qwen3-1.7B W3G128, trained from the same BF16 layer-wise checkpoint and calibration corpus as LAQuant but with E2E loss in place of the layer-wise objective) on Qwen3-1.7B; the only thing differing between the two is the supervision. For each of $16$ AIME25 responses $\times$ $6$ highest-entropy positions $\times$ $\ell \in \{4,8,12,16,20,24\}$ ($453$ valid triples after OOMs at small $\ell$) we measure $\|M_{\mathrm{E2E}} v\|/\|v\|$ at $v \in \{v_\star, v_r, g_K^{\mathrm{LW}}, g_V^{\mathrm{LW}}, g_K^{\mathrm{ctrl\text{-}E2E}}, g_V^{\mathrm{ctrl\text{-}E2E}}\}$, where $v_\star$ is the leading eigenvector of $M_{\mathrm{E2E}}$ from $\leq 20$ power-iter steps, $v_r$ a random Gaussian, and $g_{K/V}^X = \nabla_{h_\ell[p]}\|K_{\ell+1}^X(h_\ell) - K_{\ell+1}^*(h_\ell)\|^2$ is the residual direction that maximally reduces next-layer K- (resp.\ V-) cache error for student $X$. JVP via \texttt{torch.func.jvp} (exact, eager attention); VJP via reverse-mode AD; causal sequence truncation to $[:p+1]$. Entropy stratification ($H \geq 2$ nats) avoids the near-deterministic regime where $H_{\mathrm{KL}}$ is null-rank in finite precision.

\paragraph{Anisotropy of $M_{\mathrm{E2E}}$.}
Table~\ref{tab:anisotropy_overall} reports the medians: the leading eigendirection of $M_{\mathrm{E2E}}$ is $\sim$30--40$\times$ more responsive than every other direction we test, and the K/V-cache-reducing directions sit in the attenuated regime alongside random Gaussians. By contrast, $M_{\mathrm{LW}} = I$ (Proposition~\ref{prop:metrics}) assigns responsiveness exactly $1.0$ to every direction. The observed $\geq 30\times$ ratio rules out the uniform-shrinkage alternative, which would require every direction's responsiveness to be near $1.0$ relative to $v_\star$ (and would in principle be compensable by a learning-rate scaling); we observe $\leq 0.05$ at the median for every non-leading direction tested, with the pattern robust across all six $\ell$ values.

\paragraph{$\sigma_{\max}(\mathcal{J})$ and bound looseness.}
Direct power iteration on $\mathcal{J}^{\!\top}\mathcal{J}$ gives median $\sigma_{\max}(\mathcal{J}) = 40.8$ (range $9.1$--$479.2$, growing with the count of downstream layers): $\mathcal{J}$ is mildly expansive in this regime. The operator-norm bound's RHS $2 \max_i p^\star_i(1-p^\star_i) \cdot \sigma_{\max}(\mathcal{J})^2$ therefore exceeds the directly measured $\lambda_{\max}(M_{\mathrm{E2E}})$ by $4$--$5$ orders of magnitude --- the bound holds but is loose, since it assumes worst-case alignment between $\mathcal{J}$'s leading singular vector and $H_{\mathrm{KL}}$'s leading eigenvector that holds with probability $O(1/V)$ in generic alignment. The dilution claim does not depend on the bound's tightness: the direction-stratified ratio is invariant to the absolute scale of $\sigma_{\max}(\mathcal{J})$, which cancels as a common multiplier.

\paragraph{Cross-student magnitudes.}
The controlled-E2E student's residual K-cache gradient is $1.52\times$ larger than the LW student's (median $\|g_K^{\mathrm{ctrl\text{-}E2E}}\|/\|g_K^{\mathrm{LW}}\|$, IQR $1.40$--$1.69$); the V-cache ratio is $1.33$ (IQR $1.22$--$1.48$). E2E supervision left $\sim$50\% more residual K-cache error than LW under matched training, on a residual direction $M_{\mathrm{E2E}}$ attenuates $\sim$30$\times$ regardless of student ($\cos(g_K^{\mathrm{LW}}, g_K^{\mathrm{ctrl\text{-}E2E}}) = 0.67$, IQR $0.46$--$0.74$). $M_{\mathrm{LW}} = I$ avoids this failure mode entirely.

\paragraph{Top-1 agreement with the BF16 teacher.}
Table~\ref{tab:top1_agreement} reports the fraction of next-token positions where the BF16 teacher's argmax matches each model's argmax, weighted by tokens. LAQuant attains the highest agreement at $94.14\%$, followed by LW ($k=0$, LAQuant without the lookahead loss) at $93.95\%$, the original E2E baseline (ReasoningQAT) at $93.07\%$, and the controlled E2E baseline at $88.44\%$. The controlled-E2E gap (which isolates the loss-shape variable from calibration data and initialization, both shared with LAQuant) supports the dilution argument: holding everything else fixed, switching from LW to E2E supervision \emph{reduces} fidelity to the BF16 teacher's predictions, despite the E2E loss explicitly optimizing for that fidelity. The LA$-$LW gap is small in aggregate ($+0.20$ pp) but concentrates at high-entropy positions where decisions actually compete (Table~\ref{tab:top1_entropy}): negligible at $H<0.1$ ($+0.004$ to $+0.050$ pp, $\sim$56\% of tokens) but reaching $+0.77$ pp at $H\in[1,2)$ and $+2.50$ pp at $H\in[2,5)$.

\begin{table}[!ht]
\centering
\caption{Median (IQR) of $\|M_{\mathrm{E2E}} v\| / \|M_{\mathrm{E2E}} v_\star\|$ across 453 (sample, position, layer) triples on Qwen3-1.7B AIME25 traces. By Proposition~\ref{prop:metrics}, $M_{\mathrm{LW}} = I$ assigns responsiveness $1.0$ to every direction.}
\label{tab:anisotropy_overall}
\small
\setlength{\tabcolsep}{6pt}
\renewcommand{\arraystretch}{0.95}
\begin{tabular}{lc}
\toprule
direction & median (q25, q75) \\
\midrule
$v_\star$ (leading eigvec)        & $1.000$ \\
random $v_r$                      & $0.025$ ($0.017$, $0.034$) \\
$g_K^{\mathrm{LW}}$               & $0.037$ ($0.025$, $0.053$) \\
$g_K^{\mathrm{ctrl\text{-}E2E}}$  & $0.037$ ($0.026$, $0.055$) \\
$g_V^{\mathrm{LW}}$               & $0.027$ ($0.019$, $0.037$) \\
$g_V^{\mathrm{ctrl\text{-}E2E}}$  & $0.029$ ($0.020$, $0.042$) \\
\bottomrule
\end{tabular}
\end{table}

\begin{table}[!ht]
\centering
\begin{minipage}[t]{0.34\linewidth}
\centering
\caption{Token-weighted top-1 argmax agreement with the BF16 teacher on Qwen3-1.7B AIME25 traces (488{,}455 tokens, 120 traces).}
\label{tab:top1_agreement}
\small
\setlength{\tabcolsep}{4pt}
\renewcommand{\arraystretch}{0.95}
\begin{tabular}{lc}
\toprule
Model (W3G128) & Agreement \\
\midrule
LAQuant ($k{=}1$)        & \textbf{94.14\%} \\
LW ($k{=}0$)             & 93.95\% \\
E2E (ReasoningQAT)       & 93.07\% \\
E2E-Ctrl                 & 88.44\% \\
\bottomrule
\end{tabular}
\end{minipage}\hfill
\begin{minipage}[t]{0.62\linewidth}
\centering
\caption{Top-1 agreement vs the BF16 teacher, stratified by teacher entropy $H$. The LA$-$LW gap grows with $H$: at sharply peaked positions ($H<0.1$, $\sim$56\% of tokens) LA and LW are indistinguishable; at uncertain positions ($H\geq 1$) LA preserves the teacher's choice noticeably more often.}
\label{tab:top1_entropy}
\small
\setlength{\tabcolsep}{4pt}
\renewcommand{\arraystretch}{0.95}
\begin{tabular}{lrrrrr}
\toprule
$H$ range & frac. & LA & LW & E2E & LA$-$LW \\
\midrule
$[0, 0.01)$   & 45.6\% & 99.98\% & 99.97\% & 99.94\% & $+0.004$ \\
$[0.01, 0.1)$ & 10.6\% & 99.81\% & 99.76\% & 99.57\% & $+0.050$ \\
$[0.1, 0.5)$  & 15.1\% & 98.95\% & 98.75\% & 98.02\% & $+0.199$ \\
$[0.5, 1.0)$  & 16.6\% & 86.29\% & 85.92\% & 83.83\% & $+0.372$ \\
$[1.0, 2.0)$  & 11.6\% & 72.97\% & 72.20\% & 68.94\% & $\mathbf{+0.766}$ \\
$[2.0, 5.0)$  & 0.6\%  & 58.21\% & 55.70\% & 50.40\% & $\mathbf{+2.502}$ \\
\bottomrule
\end{tabular}
\end{minipage}
\end{table}

\section{Peaky Softmax Distribution on AIME25}
\label{app:peaky_softmax}

We empirically validate that the BF16 teacher's softmax distribution on LRM decoding is sharply peaked, so the spectral bound $\sigma_{\max}(H_{\mathrm{KL}}) \leq 2 \max_i p^\star_i(1-p^\star_i)$ used in Section~\ref{sec:supervision} is small in practice.

We collect 488{,}455 next-token softmax distributions from Qwen3-1.7B (BF16) decoding 120 AIME25 traces. The teacher's top-1 probability has mean $0.877$ and median $0.996$; $67.6\%$ of tokens have $\max_i p^\star_i > 0.9$ and $53.4\%$ have $\max_i p^\star_i > 0.99$. By Lemma~\ref{lem:fisher_bound}, for the latter group $\sigma_{\max}(H_{\mathrm{KL}}) \leq 2 \cdot 0.99 \cdot 0.01 \approx 0.02$, more than an order of magnitude below the worst-case bound of $1/2$ (a factor of $\sim 25\times$).

\section{Top-$k$ Sensitivity of Hessian-Subspace Overlap}
\label{app:topk_sensitivity}

The principal-angle diagnostic of Section~\ref{sec:calibdata} uses the average top-$k$ cosine $\bar c_k$ between $H_{\mathrm{act}}^{\mathrm{cal}}$ and $H_{\mathrm{act}}^{\mathrm{eval}}$, with $k=64$ in Figure~\ref{fig:hessian}(b) ($\sim$3\% of the $d{=}2048$ eigenspace). We sweep $k \in \{4, 8, 16, 32, 64, 128, 256\}$ to confirm the calibration ranking does not depend on this choice.

Following the Figure~\ref{fig:hessian}(b) setup, we compute $\bar c_k$ across $112$ (layer, hook) keys ($28$ transformer layers $\times$ $4$ attention/MLP hook points) on Qwen3-1.7B activations, between AIME25 and four candidate calibration corpora. Table~\ref{tab:topk_sensitivity} reports the mean cosine overlap.

\paragraph{Setup.}
For each of $4$ hook sites per transformer layer on Qwen3-1.7B ($28$ layers $\times$ $4$ sites $=112$ keys; the four sites are the inputs to \texttt{q\_proj} (attention input, shared by q/k/v), \texttt{o\_proj} (attention output), \texttt{gate\_proj} (MLP input, shared by gate/up), and \texttt{down\_proj} (MLP hidden)), we collect activations from $128$ sequences $\times$ $2{,}048$ tokens per corpus and accumulate the uncentered second-moment matrix $H_{\mathrm{act}} = \sum_t \phi_t \phi_t^{\top}$. Eigendecomposition uses \texttt{torch.linalg.eigh} truncated to the top $256$ eigenvectors per (corpus, key); per-pair cosines are arithmetic-meaned across the $112$ keys. Table~\ref{tab:topk_sensitivity} reports the resulting mean cosine overlap between AIME25 and four candidate calibration corpora.

The ranking $\text{OpenR1-Math} > \text{OpenThoughts3} > \text{ReasoningQAT} > \text{WikiText2+C4+RedPajama}$ holds at every $k \in [4, 256]$, and the reasoning-vs-generic gap stays $\geq 0.21$ cosine units throughout. A mild U-shape minimized near $k \approx 32$--$64$ reflects that small $k$ over-weights a few high-variance directions, while large $k$ admits low-variance directions that are nearly uniformly shared across corpora; $k=64$ sits at the discriminating regime.

\begin{table}[!ht]
\centering
\caption{Mean top-$k$ principal-angle cosine $\bar c_k$ between four calibration corpora and the AIME25 deployment distribution on Qwen3-1.7B, averaged across 112 (layer, hook) keys. W2 and RP denote WikiText2 and RedPajama, respectively.}
\label{tab:topk_sensitivity}
\small
\setlength{\tabcolsep}{4pt}
\renewcommand{\arraystretch}{0.95}
\begin{tabular}{lccccccc}
\toprule
Calibration corpus & $k{=}4$ & $k{=}8$ & $k{=}16$ & $k{=}32$ & $k{=}64$ & $k{=}128$ & $k{=}256$ \\
\midrule
OpenR1-Math (LAQuant)  & 0.947 & 0.915 & 0.886 & 0.864 & \textbf{0.846} & 0.836 & 0.835 \\
OpenThoughts3          & 0.757 & 0.721 & 0.701 & 0.691 & \textbf{0.686} & 0.695 & 0.716 \\
ReasoningQAT mix       & 0.711 & 0.689 & 0.674 & 0.669 & \textbf{0.666} & 0.676 & 0.697 \\
W2+C4+RP (generic)     & 0.476 & 0.468 & 0.450 & 0.451 & \textbf{0.455} & 0.478 & 0.517 \\
\bottomrule
\end{tabular}
\end{table}

\section{Details of Evaluation Settings}
\label{app:eval_settings}
\paragraph{Evaluation Settings.}
We use lighteval~\cite{lighteval} for reasoning and lm\_eval~\cite{lmeval} for non-reasoning benchmarks, both on vLLM v0.10.1.1~\cite{vllm}. We follow each model's original sampling: temperature 0.6 with \texttt{top\_p}=0.95, plus \texttt{top\_k}=20 for Qwen3. 

For AIME24, AIME25, GPQA, and LSAT-AR, we generate 16 responses for each problem and compute Pass@1 from the generated samples. For MMLU-Pro and LiveCodeBench, we generate a single response per problem and compute Pass@1 accordingly. Due to the limited number of problems in AIME24 and AIME25, which leads to relatively high statistical variability, we report results averaged over four random seeds, 0, 1, 2, and 3, to obtain more reliable estimates. For LiveCodeBench, we also report the average over three random seeds, 0, 1, and 42. For all other reasoning tasks, we use a fixed random seed of 42.

For DeepSeek-R1-distilled Llama-3.1, a known issue has been reported in which the model may skip the thinking process even when evaluated in thinking mode. To avoid this issue, we explicitly prepend the \texttt{<think>} token during all evaluations\footnote{\url{https://huggingface.co/deepseek-ai/DeepSeek-R1-Distill-Llama-8B\#usage-recommendations}}.
\clearpage
\section{Additional Evaluation Results}
\subsection{Reasoning Tasks}
\label{app:full_reason_tasks}

\paragraph{Sample budget and variability.}
For AIME24, AIME25, GPQA, and LSAT-AR, we generate $16$ responses per problem to estimate Pass@1. Since AIME24 and AIME25 each contain only $30$ competition problems, their estimates exhibit relatively high run-to-run variability. We therefore additionally average the results over four random seeds ($0, 1, 2, 3$), yielding a total of $64$ samples per problem for these benchmarks. Table~\ref{tab:res_main_reasoning_std} reports the mean and standard deviation of Pass@1 for each method, allowing readers to assess the statistical variability together with the central estimate.

\paragraph{ReasoningQAT (R-QAT) training data.}
R-QAT uses a public configuration whose calibration data differ from those of LAQuant. In phase 1, corresponding to layer-wise QAT, R-QAT uses a mixture of reasoning data (OpenThoughts-3) and generic calibration data with $4{,}096$ samples and a maximum sequence length of $2{,}048$. In phase 2, corresponding to end-to-end QAT, it further uses reasoning-only data from OpenThoughts-3 with $32{,}768$ samples and a maximum sequence length of $8{,}192$. Thus, although the calibration corpus is not exactly matched to ours, R-QAT is trained with a substantially larger number of reasoning-domain training tokens than LAQuant. Because reproducing both phases under a fully controlled calibration setup is non-trivial, we evaluate R-QAT using its prescribed public configuration in all comparisons. We instead provide the matched-calibration comparison through ParoQuant++, which uses the same reasoning-domain corpus and training data budget as LAQuant.

\begin{table}[!ht]
\centering
\caption{Pass@N$^{\uparrow}$ and standard deviations on \textbf{AIME}, \textbf{GPQA Diamond}, and \textbf{LSAT-AR} under 3-bit and 4-bit quantization. R-QAT and ParoQ denote ReasoningQAT and ParoQuant, respectively. Results on AIME benchmarks are averaged over four random seeds, 0, 1, 2, and 3.}
\label{tab:res_main_reasoning_std}

\begin{subtable}{\textwidth}
\centering
\caption{W3G128 Pass@N accuracy on reasoning tasks.}
\label{tab:res_w3_reason_std}
\scriptsize
\setlength{\tabcolsep}{6pt}
\renewcommand{\arraystretch}{0.9}
\resizebox{\textwidth}{!}{%
\begin{tabular}{lllcccccccc}
\toprule[1.2pt]
\multirow{2}{*}{\textbf{Model}} 
& \multirow{2}{*}{\textbf{Method}} 
& \multirow{2}{*}{\textbf{Dataset}}
& \multicolumn{2}{c}{\textbf{AIME24}}
& \multicolumn{2}{c}{\textbf{AIME25}}
& \multicolumn{2}{c}{\textbf{GPQA}}
& \multicolumn{2}{c}{\textbf{LSAT}} \\
\cmidrule(lr){4-5} \cmidrule(lr){6-7} \cmidrule(lr){8-9} \cmidrule(lr){10-11}
& & 
& \textbf{Pass@1} & \textbf{Pass@64}
& \textbf{Pass@1} & \textbf{Pass@64}
& \textbf{Pass@1} & \textbf{Pass@16}
& \textbf{Pass@1} & \textbf{Pass@16} \\
\midrule

\multirow{7}{*}{Qwen3-1.7B}
& FP16  &        & 45.84 $\pm$ 6.20 & 83.30 & 35.94 $\pm$ 4.58 & 73.30 & 40.70 $\pm$ 2.08 & 78.28 & 64.10 $\pm$ 2.53 & 95.65 \\
\cmidrule(lr){2-11}
& GPTQ  & Generic & 0.00 $\pm$ 0.00 & 0.00 & 0.00 $\pm$ 0.00 & 0.00 & 25.03 $\pm$ 3.05 & 81.82 & 18.78 $\pm$ 2.72 & 76.96 \\
& GPTQ++  & OpenR1  & 17.14 $\pm$ 4.48 & 60.00 & 22.55 $\pm$ 4.35 & 46.70 & 28.88 $\pm$ 1.77 & 72.22 & 34.40 $\pm$ 2.36 & 80.43 \\
& R-QAT & Mixed   & 25.52 $\pm$ 6.26 & 73.30 & 23.70 $\pm$ 4.79 & 60.00 & 31.03 $\pm$ 2.68 & 75.76 & 47.69 $\pm$ 2.09 & 95.22 \\
& ParoQ & Generic & 7.24 $\pm$ 4.77 & 53.30 & 16.88 $\pm$ 4.40 & 46.70 & 30.49 $\pm$ 3.19 & 76.26 & 36.10 $\pm$ 2.68 & 81.30 \\
& ParoQ++ & OpenR1  & 33.65 $\pm$ 5.93 & 73.30 & 28.23 $\pm$ 4.28 & 56.70 & 30.81 $\pm$ 2.61 & 76.26 & 49.30 $\pm$ 2.38 & 93.48 \\
& LAQuant  & OpenR1  & 34.64 $\pm$ 6.10 & 80.00 & 29.33 $\pm$ 4.87 & 60.00 & 31.79 $\pm$ 2.16 & 75.76 & 54.86 $\pm$ 2.11 & 93.91 \\
\midrule

\multirow{7}{*}{Qwen3-4B}
& FP16  &        & 73.23 $\pm$ 4.52 & 90.00 & 63.75 $\pm$ 5.65 & 86.67 & 55.59 $\pm$ 2.20 & 82.32 & 81.14 $\pm$ 1.21 & 95.22 \\
\cmidrule(lr){2-11}
& GPTQ  & Generic & 12.92 $\pm$ 4.46 & 36.67 & 12.76 $\pm$ 3.59 & 40.00 & 33.65 $\pm$ 2.56 & 74.75 & 40.41 $\pm$ 1.78 & 86.09 \\
& GPTQ++  & OpenR1  & 54.95 $\pm$ 5.73 & 76.67 & 51.21 $\pm$ 4.67 & 80.00 & 41.60 $\pm$ 2.02 & 76.77 & 67.45 $\pm$ 1.40 & 90.87 \\
& R-QAT & Mixed   & 43.75 $\pm$ 6.04 & 83.33 & 35.31 $\pm$ 4.74 & 76.67 & 46.21 $\pm$ 2.91 & 78.79 & 67.64 $\pm$ 2.67 & 94.78 \\
& ParoQ & Generic & 41.30 $\pm$ 5.50 & 73.33 & 39.58 $\pm$ 5.60 & 80.00 & 44.41 $\pm$ 2.27 & 74.24 & 68.94 $\pm$ 1.86 & 90.00 \\
& ParoQ++ & OpenR1  & 59.74 $\pm$ 6.73 & 83.33 & 52.76 $\pm$ 4.70 & 83.33 & 46.17 $\pm$ 2.07 & 80.23 & 77.07 $\pm$ 1.85 & 94.35 \\
& LAQuant  & OpenR1  & 62.92 $\pm$ 5.39 & 90.00 & 54.69 $\pm$ 5.56 & 86.67 & 46.21 $\pm$ 2.11 & 77.78 & 77.15 $\pm$ 1.94 & 96.09 \\
\midrule

\multirow{7}{*}{Qwen3-8B}
& FP16  &        & 74.95 $\pm$ 4.24 & 93.33 & 67.66 $\pm$ 5.81 & 86.67 & 58.78 $\pm$ 2.49 & 87.88 & 85.57 $\pm$ 1.24 & 95.65 \\
\cmidrule(lr){2-11}
& GPTQ  & Generic & 44.48 $\pm$ 6.46 & 73.33 & 40.68 $\pm$ 6.09 & 76.70 & 48.17 $\pm$ 1.69 & 79.80 & 69.92 $\pm$ 2.41 & 92.61 \\
& GPTQ++  & OpenR1  & 68.08 $\pm$ 4.96 & 86.67 & 59.43 $\pm$ 6.07 & 83.30 & 51.20 $\pm$ 1.80 & 82.83 & 82.28 $\pm$ 1.93 & 96.09 \\
& R-QAT & Mixed   & 53.75 $\pm$ 5.83 & 83.33 & 43.91 $\pm$ 6.13 & 86.70 & 48.61 $\pm$ 2.08 & 83.84 & 77.23 $\pm$ 2.11 & 94.78 \\
& ParoQ & Generic & 59.53 $\pm$ 5.50 & 80.00 & 52.92 $\pm$ 6.13 & 83.30 & 52.30 $\pm$ 1.85 & 82.83 & 78.53 $\pm$ 1.34 & 95.65 \\
& ParoQ++ & OpenR1  & 67.19 $\pm$ 5.76 & 90.00 & 58.39 $\pm$ 6.44 & 83.30 & 53.09 $\pm$ 2.60 & 86.36 & 79.48 $\pm$ 1.87 & 96.96 \\
& LAQuant  & OpenR1  & 68.33 $\pm$ 4.75 & 93.33 & 61.30 $\pm$ 6.10 & 86.67 & 53.66 $\pm$ 2.02 & 85.86 & 83.80 $\pm$ 1.94 & 96.09 \\
\midrule

\multirow{7}{*}{R1-Distill-Llama-8B}
& FP16  &        & 43.44 $\pm$ 6.45 & 80.00 & 31.20 $\pm$ 5.32 & 70.00 & 48.17 $\pm$ 2.17 & 86.87 & 59.78 $\pm$ 2.97 & 96.52 \\
\cmidrule(lr){2-11}
& GPTQ  & Generic & 15.45 $\pm$ 6.21 & 73.33 & 19.64 $\pm$ 4.03 & 60.00 & 37.97 $\pm$ 2.57 & 85.35 & 42.47 $\pm$ 3.09 & 95.22 \\
& GPTQ++  & OpenR1  & 25.84 $\pm$ 6.26 & 83.33 & 25.73 $\pm$ 5.02 & 66.67 & 37.78 $\pm$ 2.48 & 84.34 & 46.36 $\pm$ 2.73 & 93.48 \\
& R-QAT & Mixed   & 30.16 $\pm$ 6.78 & 83.33 & 27.50 $\pm$ 4.44 & 66.67 & 40.50 $\pm$ 2.46 & 82.83 & 48.48 $\pm$ 2.32 & 94.78 \\
& ParoQ & Generic & 0.89 $\pm$ 1.81 & 13.33 & 0.05 $\pm$ 0.13 & 3.33 & 32.99 $\pm$ 2.66 & 88.38 & 29.65 $\pm$ 2.32 & 93.48 \\
& ParoQ++ & OpenR1  & 30.63 $\pm$ 6.79 & 76.67 & 25.57 $\pm$ 4.90 & 56.67 & 40.37 $\pm$ 2.38 & 86.36 & 48.89 $\pm$ 2.94 & 95.65 \\
& LAQuant  & OpenR1  & 31.46 $\pm$ 7.34 & 83.33 & 28.54 $\pm$ 3.80 & 70.00 & 43.31 $\pm$ 2.27 & 84.85 & 50.98 $\pm$ 2.43 & 92.61 \\
\bottomrule[1.2pt]
\end{tabular}%
}
\end{subtable}

\vspace{1em}

\begin{subtable}{\textwidth}
\centering
\caption{W4G128 Pass@N on reasoning tasks.}
\label{tab:res_w4_reason_std}
\scriptsize
\setlength{\tabcolsep}{6pt}
\renewcommand{\arraystretch}{0.9}
\resizebox{\textwidth}{!}{%
\begin{tabular}{lllcccccccc}
\toprule[1.2pt]
\multirow{2}{*}{\textbf{Model}} 
& \multirow{2}{*}{\textbf{Method}} 
& \multirow{2}{*}{\textbf{Dataset}}
& \multicolumn{2}{c}{\textbf{AIME24}}
& \multicolumn{2}{c}{\textbf{AIME25}}
& \multicolumn{2}{c}{\textbf{GPQA}}
& \multicolumn{2}{c}{\textbf{LSAT}} \\
\cmidrule(lr){4-5} \cmidrule(lr){6-7} \cmidrule(lr){8-9} \cmidrule(lr){10-11}
& & 
& \textbf{Pass@1} & \textbf{Pass@64}
& \textbf{Pass@1} & \textbf{Pass@64}
& \textbf{Pass@1} & \textbf{Pass@16}
& \textbf{Pass@1} & \textbf{Pass@16} \\
\midrule

\multirow{7}{*}{Qwen3-1.7B}
& FP16  &        & 45.84 $\pm$ 6.20 & 83.33 & 35.94 $\pm$ 4.58 & 73.33 & 40.70 $\pm$ 2.08 & 78.28 & 64.10 $\pm$ 2.53 & 95.65 \\
\cmidrule(lr){2-11}
& GPTQ  & Generic & 24.17 $\pm$ 5.60 & 60.00 & 25.47 $\pm$ 4.13 & 53.33 & 33.14 $\pm$ 2.06 & 77.78 & 48.29 $\pm$ 2.42 & 86.52 \\
& GPTQ++  & OpenR1  & 38.96 $\pm$ 6.78 & 80.00 & 33.33 $\pm$ 4.28 & 66.67 & 34.60 $\pm$ 2.24 & 72.73 & 61.03 $\pm$ 2.36 & 91.74 \\
& R-QAT & Mixed   & 32.61 $\pm$ 6.49 & 83.33 & 27.29 $\pm$ 5.10 & 66.67 & 35.70 $\pm$ 2.32 & 76.77 & 53.48 $\pm$ 1.74 & 93.04 \\
& ParoQ & Generic & 39.53 $\pm$ 5.82 & 80.00 & 32.55 $\pm$ 4.63 & 66.67 & 35.86 $\pm$ 2.89 & 73.23 & 58.30 $\pm$ 2.45 & 92.17 \\
& ParoQ++ & OpenR1  & 41.82 $\pm$ 5.94 & 80.00 & 32.71 $\pm$ 4.58 & 66.67 & 36.30 $\pm$ 3.09 & 72.73 & 60.52 $\pm$ 2.33 & 91.74 \\
& LAQuant  & OpenR1  & 42.35 $\pm$ 6.52 & 80.00 & 35.47 $\pm$ 4.54 & 73.33 & 36.43 $\pm$ 1.64 & 76.77 & 61.14 $\pm$ 2.23 & 91.74 \\
\midrule

\multirow{7}{*}{Qwen3-4B}
& FP16  &        & 73.23 $\pm$ 4.52 & 90.00 & 63.75 $\pm$ 5.65 & 86.67 & 55.59 $\pm$ 2.20 & 82.32 & 81.14 $\pm$ 1.21 & 95.22 \\
\cmidrule(lr){2-11}
& GPTQ  & Generic & 63.44 $\pm$ 4.75 & 80.00 & 58.67 $\pm$ 5.40 & 86.67 & 51.61 $\pm$ 1.77 & 82.32 & 78.10 $\pm$ 2.18 & 92.61 \\
& GPTQ++  & OpenR1  & 68.33 $\pm$ 4.90 & 90.00 & 62.20 $\pm$ 5.19 & 90.00 & 51.86 $\pm$ 2.37 & 77.78 & 78.61 $\pm$ 1.25 & 93.91 \\
& R-QAT & Mixed   & 64.64 $\pm$ 5.56 & 86.67 & 57.66 $\pm$ 5.56 & 80.00 & 52.15 $\pm$ 1.27 & 78.79 & 78.45 $\pm$ 1.89 & 93.91 \\
& ParoQ & Generic & 67.29 $\pm$ 5.63 & 86.67 & 60.78 $\pm$ 5.12 & 90.00 & 52.46 $\pm$ 2.16 & 79.29 & 79.43 $\pm$ 1.77 & 93.48 \\
& ParoQ++ & OpenR1  & 70.37 $\pm$ 4.83 & 90.00 & 61.72 $\pm$ 5.28 & 83.33 & 53.44 $\pm$ 2.42 & 83.84 & 80.87 $\pm$ 1.34 & 95.22 \\
& LAQuant  & OpenR1  & 70.42 $\pm$ 4.81 & 90.00 & 62.66 $\pm$ 5.79 & 86.67 & 53.60 $\pm$ 1.65 & 81.31 & 81.03 $\pm$ 2.34 & 94.78 \\
\midrule

\multirow{7}{*}{Qwen3-8B}
& FP16  &        & 74.95 $\pm$ 4.24 & 93.33 & 67.66 $\pm$ 5.81 & 86.67 & 58.78 $\pm$ 2.49 & 87.88 & 85.57 $\pm$ 1.24 & 95.65 \\
\cmidrule(lr){2-11}
& GPTQ  & Generic & 71.56 $\pm$ 5.34 & 83.33 & 65.21 $\pm$ 4.86 & 83.33 & 57.89 $\pm$ 2.35 & 82.32 & 82.31 $\pm$ 1.42 & 96.09 \\
& GPTQ++  & OpenR1  & 72.76 $\pm$ 5.02 & 93.33 & 65.16 $\pm$ 5.07 & 86.67 & 57.83 $\pm$ 2.11 & 85.35 & 83.89 $\pm$ 1.17 & 96.52 \\
& R-QAT & Mixed   & 68.44 $\pm$ 4.97 & 93.33 & 59.48 $\pm$ 6.46 & 83.33 & 58.49 $\pm$ 2.07 & 86.87 & 85.24 $\pm$ 1.25 & 96.96 \\
& ParoQ & Generic & 73.39 $\pm$ 4.70 & 93.33 & 64.52 $\pm$ 5.58 & 86.67 & 56.85 $\pm$ 2.26 & 84.34 & 84.43 $\pm$ 1.43 & 95.22 \\
& ParoQ++ & OpenR1  & 73.81 $\pm$ 5.17 & 90.00 & 64.27 $\pm$ 5.23 & 90.00 & 56.66 $\pm$ 2.38 & 83.84 & 83.99 $\pm$ 1.30 & 96.09 \\
& LAQuant  & OpenR1  & 74.17 $\pm$ 4.48 & 93.33 & 66.98 $\pm$ 4.78 & 86.67 & 57.23 $\pm$ 2.53 & 86.36 & 85.24 $\pm$ 1.77 & 96.52 \\
\midrule

\multirow{7}{*}{R1-Distill-Llama-8B}
& FP16  &        & 43.44 $\pm$ 6.45 & 80.00 & 31.20 $\pm$ 5.32 & 70.00 & 48.17 $\pm$ 2.17 & 86.87 & 59.78 $\pm$ 2.97 & 96.52 \\
\cmidrule(lr){2-11}
& GPTQ  & Generic & 38.44 $\pm$ 6.95 & 80.00 & 29.85 $\pm$ 5.05 & 66.67 & 46.65 $\pm$ 2.38 & 87.88 & 56.66 $\pm$ 2.68 & 95.22 \\
& GPTQ++  & OpenR1  & 39.53 $\pm$ 5.34 & 83.33 & 29.74 $\pm$ 4.42 & 63.33 & 45.30 $\pm$ 2.42 & 84.85 & 58.40 $\pm$ 2.17 & 95.22 \\
& R-QAT & Mixed   & 40.63 $\pm$ 5.94 & 80.00 & 29.53 $\pm$ 4.63 & 63.33 & 46.09 $\pm$ 2.09 & 84.85 & 57.01 $\pm$ 1.95 & 94.78 \\
& ParoQ & Generic & 39.12 $\pm$ 5.98 & 83.33 & 28.54 $\pm$ 5.13 & 70.00 & 46.34 $\pm$ 2.71 & 85.86 & 54.81 $\pm$ 3.19 & 98.26 \\
& ParoQ++ & OpenR1  & 38.86 $\pm$ 7.28 & 83.33 & 27.60 $\pm$ 4.66 & 66.67 & 46.46 $\pm$ 2.88 & 83.33 & 56.66 $\pm$ 2.77 & 96.96 \\
& LAQuant  & OpenR1  & 41.41 $\pm$ 5.90 & 83.33 & 31.77 $\pm$ 5.41 & 73.33 & 45.55 $\pm$ 2.88 & 86.87 & 58.67 $\pm$ 1.71 & 95.22 \\
\bottomrule[1.2pt]
\end{tabular}%
}
\end{subtable}

\end{table}
\begin{table}[!ht]
\centering
\caption{Pass@1$^{\uparrow}$ on \textbf{MMLU-Pro} and \textbf{LiveCodeBench} under 3-bit and 4-bit quantization. R-QAT and ParoQ denote ReasoningQAT and ParoQuant, respectively. Results on LiveCodeBench benchmarks are averaged over three random seeds, 0, 1, and 42.}
\label{tab:res_mmlu_lcb_main}

\begin{subtable}[t]{0.5\textwidth}
\centering
\caption{W3G128 Pass@1 results.}
\label{tab:res_mmlu_lcb_w3}
\scriptsize
\setlength{\tabcolsep}{4pt}
\renewcommand{\arraystretch}{0.9}
\resizebox{\linewidth}{!}{%
\begin{tabular}{lllcc}
\toprule[1.2pt]
\textbf{Model} & \textbf{Method} & \textbf{Dataset} & \textbf{MMLU-Pro} & \textbf{LCB} \\
\midrule
\multirow{7}{*}{Qwen3-1.7B}
& FP16   &       & 57.31 & 33.71 \\
\cmidrule(lr){2-5}
& GPTQ   & Generic & 22.34 & 0.00 \\
& GPTQ++   & OpenR1  & 35.74 & 3.73 \\
& R-QAT  & Mixed   & 47.48 & 19.90 \\
& ParoQ  & Generic & 43.79 & 12.94 \\
& ParoQ++  & OpenR1  & 45.84 & 19.40 \\
& LAQuant & OpenR1 & 47.94 & 21.02 \\
\midrule
\multirow{7}{*}{Qwen3-4B}
& FP16   &       & 70.23 & 54.48 \\
\cmidrule(lr){2-5}
& GPTQ   & Generic & 53.32 & 2.99 \\
& GPTQ++   & OpenR1  & 60.61 & 29.97 \\
& R-QAT  & Mixed   & 63.24 & 37.31 \\
& ParoQ  & Generic & 63.41 & 38.31 \\
& ParoQ++  & OpenR1  & 62.41 & 45.03 \\
& LAQuant & OpenR1 & 62.63 & 45.15 \\
\midrule
\multirow{7}{*}{Qwen3-8B}
& FP16   &       & 73.44 & 58.46 \\
\cmidrule(lr){2-5}
& GPTQ   & Generic & 65.71 & 37.94 \\
& GPTQ++   & OpenR1  & 66.29 & 46.64 \\
& R-QAT  & Mixed   & 68.33 & 44.16 \\
& ParoQ  & Generic & 68.93 & 44.65 \\
& ParoQ++  & OpenR1  & 63.32 & 51.12 \\
& LAQuant & OpenR1 & 66.86 & 51.74 \\
\midrule
\multirow{7}{*}{R1-Distill-Llama-8B}
& FP16   &       & 58.84 & 36.32 \\
\cmidrule(lr){2-5}
& GPTQ   & Generic & 47.10 & 24.13 \\
& GPTQ++   & OpenR1  & 46.75 & 26.49 \\
& R-QAT  & Mixed   & 51.05 & 29.11 \\
& ParoQ  & Generic & 44.29 & 1.12 \\
& ParoQ++  & OpenR1  & 51.44 & 29.47 \\
& LAQuant & OpenR1 & 51.70 & 31.22 \\
\bottomrule[1.2pt]
\end{tabular}%
}
\end{subtable}
\hfill
\begin{subtable}[t]{0.5\textwidth}
\centering
\caption{W4G128 Pass@1 results.}
\label{tab:res_mmlu_lcb_w4}
\scriptsize
\setlength{\tabcolsep}{4pt}
\renewcommand{\arraystretch}{0.9}
\resizebox{\linewidth}{!}{%
\begin{tabular}{lllcc}
\toprule[1.2pt]
\textbf{Model} & \textbf{Method} & \textbf{Dataset} & \textbf{MMLU-Pro} & \textbf{LCB} \\
\midrule
\multirow{7}{*}{Qwen3-1.7B}
& FP16   &       & 57.31 & 33.71 \\
\cmidrule(lr){2-5}
& GPTQ   & Generic & 50.49 & 20.40 \\
& GPTQ++   & OpenR1  & 54.21 & 28.48 \\
& R-QAT  & Mixed   & 54.44 & 25.87 \\
& ParoQ  & Generic & 54.58 & 28.61 \\
& ParoQ++  & OpenR1  & 55.44 & 28.98 \\
& LAQuant & OpenR1 & 55.58 & 31.84 \\
\midrule
\multirow{7}{*}{Qwen3-4B}
& FP16   &       & 70.23 & 54.48 \\
\cmidrule(lr){2-5}
& GPTQ   & Generic & 68.03 & 48.75 \\
& GPTQ++   & OpenR1  & 68.00 & 50.75 \\
& R-QAT  & Mixed   & 68.28 & 46.76 \\
& ParoQ  & Generic & 69.41 & 50.00 \\
& ParoQ++  & OpenR1  & 68.73 & 51.62 \\
& LAQuant & OpenR1 & 68.41 & 52.74 \\
\midrule
\multirow{7}{*}{Qwen3-8B}
& FP16   &       & 73.44 & 58.46 \\
\cmidrule(lr){2-5}
& GPTQ   & Generic & 72.67 & 52.99 \\
& GPTQ++   & OpenR1  & 72.33 & 55.10 \\
& R-QAT  & Mixed   & 72.65 & 51.62 \\
& ParoQ  & Generic & 72.43 & 55.47 \\
& ParoQ++  & OpenR1  & 72.59 & 55.72 \\
& LAQuant & OpenR1 & 71.52 & 56.59 \\
\midrule
\multirow{7}{*}{R1-Distill-Llama-8B}
& FP16   &       & 58.84 & 36.32 \\
\cmidrule(lr){2-5}
& GPTQ   & Generic & 56.78 & 34.21 \\
& GPTQ++   & OpenR1  & 56.30 & 34.45 \\
& R-QAT  & Mixed   & 55.96 & 34.33 \\
& ParoQ  & Generic & 56.69 & 33.58 \\
& ParoQ++  & OpenR1  & 56.77 & 33.58 \\
& LAQuant & OpenR1 & 54.98 & 35.70 \\
\bottomrule[1.2pt]
\end{tabular}%
}
\end{subtable}
\end{table}
\clearpage
\subsection{General Tasks}
\label{app:full_general_tasks}
\begin{table}[!ht]
\centering
\caption{Detailed zero-shot accuracy$^{\uparrow}$ results on general tasks under 3-bit and 4-bit quantization.}
\label{tab:appendix_zeroshot_general}
\small
\setlength{\tabcolsep}{6pt}
\renewcommand{\arraystretch}{0.6}
\begin{tabular}{lllccccc}
\toprule[1.2pt]
\textbf{Model} & \textbf{Method} & \textbf{Bits} 
& \textbf{ARC-C} & \textbf{ARC-E} & \textbf{BoolQ} & \textbf{Hella.} & \textbf{Avg.} \\
\midrule

\multirow{7}{*}{Qwen3-1.7B}
& FP16  & 16 & 41.30 & 73.36 & 78.99 & 46.15 & 59.95 \\
\cmidrule(lr){2-8}
& GPTQ  & 3  & 33.28 & 66.96 & 62.69 & 43.94 & 51.72 \\
& ParoQ & 3  & 40.53 & 73.95 & 73.00 & 46.64 & 58.53 \\
& LAQuant  & 3  & 39.16 & 72.39 & 76.51 & 46.37 & 58.61 \\
\cmidrule(lr){2-8}
& GPTQ  & 4  & 38.74 & 71.89 & 77.58 & 47.57 & 58.95 \\
& ParoQ & 4  & 41.98 & 74.87 & 77.46 & 48.67 & 60.75 \\
& LAQuant  & 4  & 41.72 & 74.58 & 76.06 & 48.43 & 60.20 \\

\midrule
\multirow{7}{*}{Qwen3-4B}
& FP16  & 16 & 48.12 & 79.17 & 83.24 & 54.57 & 66.28 \\
\cmidrule(lr){2-8}
& GPTQ  & 3  & 42.15 & 74.03 & 73.40 & 50.86 & 60.11 \\
& ParoQ & 3  & 49.23 & 80.55 & 81.59 & 52.67 & 66.01 \\
& LAQuant  & 3  & 48.04 & 80.43 & 80.03 & 53.05 & 65.39 \\
\cmidrule(lr){2-8}
& GPTQ  & 4  & 48.46 & 78.83 & 81.74 & 54.11 & 65.79 \\
& ParoQ & 4  & 50.00 & 81.06 & 84.04 & 54.35 & 67.36 \\
& LAQuant  & 4  & 48.38 & 79.59 & 83.45 & 54.33 & 66.44 \\

\midrule
\multirow{7}{*}{Qwen3-8B}
& FP16  & 16 & 52.65 & 81.78 & 82.60 & 58.78 & 68.95 \\
\cmidrule(lr){2-8}
& GPTQ  & 3  & 50.94 & 82.24 & 80.49 & 55.69 & 67.34 \\
& ParoQ & 3  & 52.05 & 81.57 & 83.79 & 57.01 & 68.61 \\
& LAQuant  & 3  & 53.50 & 82.62 & 84.22 & 57.07 & 69.35 \\
\cmidrule(lr){2-8}
& GPTQ  & 4  & 54.10 & 82.66 & 82.75 & 58.41 & 69.48 \\
& ParoQ & 4  & 52.39 & 82.66 & 84.50 & 58.57 & 69.32 \\
& LAQuant  & 4  & 53.50 & 82.53 & 82.72 & 58.51 & 69.32 \\

\midrule
\multirow{7}{*}{Llama3.1-8B}
& FP16  & 16 & 51.02 & 81.48 & 82.20 & 60.00 & 68.68 \\
\cmidrule(lr){2-8}
& GPTQ  & 3  & 44.28 & 76.85 & 74.89 & 56.29 & 63.08 \\
& ParoQ & 3  & 45.99 & 78.24 & 80.15 & 57.67 & 65.51 \\
& LAQuant  & 3  & 45.99 & 78.37 & 80.61 & 57.57 & 65.64 \\
\cmidrule(lr){2-8}
& GPTQ  & 4  & 50.19 & 81.48 & 82.02 & 59.31 & 68.25 \\
& ParoQ & 4  & 50.09 & 81.14 & 81.59 & 59.65 & 68.12 \\
& LAQuant  & 4  & 50.17 & 81.89 & 82.14 & 59.25 & 68.36 \\

\bottomrule[1.2pt]
\end{tabular}
\end{table}
\subsection{Decoding Throughput}
\label{app:full_decode_throughput}
\begin{table}[!ht]
\centering
\caption{Decoding (batch size = 1) throughput (tokens/sec) across different GPU platforms.}
\label{tab:decoding_throughput_all}
\small
\setlength{\tabcolsep}{4pt}
\renewcommand{\arraystretch}{0.88} 
\begin{tabular}{lccccc}
\toprule[1.2pt]
\textbf{DGX A100} & \textbf{Bits} & \textbf{Qwen3-1.7B} & \textbf{Qwen3-4B} & \textbf{Qwen3-8B} & \textbf{Llama3.1-8B} \\
\midrule
FP16       & 16 & 207.9         & 113.8         & 72.8          & 74.93 \\
ParoQuant  & 4  & 217.4 (x1.05) & 140.2 (x1.23) & 116.4 (x1.60) & 126.59 (x1.69) \\
LAQuant       & 4  & 260.8 (x1.25) & 163.1 (x1.43) & 131.4 (x1.81) & 143.21 (x1.91) \\
ParoQuant  & 3  & 285.4 (x1.37) & 214.9 (x1.89) & 203.8 (x2.80) & 233.73 (x3.12) \\
LAQuant       & 3  & 359.3 (x1.73) & 274.9 (x2.42) & 256.0 (x3.52) & 294.37 (x3.93) \\

\midrule
\textbf{RTX 4090} & \textbf{Bits} & \textbf{Qwen3-1.7B} & \textbf{Qwen3-4B} & \textbf{Qwen3-8B} & \textbf{Llama3.1-8B} \\
\midrule
FP16       & 16 & 187.3         & 92.5          & 55.0          & 56.30 \\
ParoQuant  & 4  & 271.6 (x1.45) & 168.4 (x1.82) & 114.6 (x2.08) & 122.66 (x2.18) \\
LAQuant       & 4  & 302.1 (x1.61) & 183.5 (x1.99) & 121.8 (x2.21) & 129.71 (x2.30) \\
ParoQuant  & 3  & 365.8 (x1.95) & 279.6 (x3.02) & 243.2 (x4.42) & 282.94 (x5.30) \\
LAQuant       & 3  & 426.0 (x2.27) & 322.1 (x3.48) & 276.6 (x5.03) & 322.11 (x5.72) \\

\midrule
\textbf{RTX 6000 Ada} & \textbf{Bits} & \textbf{Qwen3-1.7B} & \textbf{Qwen3-4B} & \textbf{Qwen3-8B} & \textbf{Llama3.1-8B} \\
\midrule
FP16       & 16 & 187.1         & 90.8          & 52.5          & 53.55 \\
ParoQuant  & 4  & 279.7 (x1.49) & 174.4 (x1.92) & 112.6 (x2.14) & 120.93 (x2.26) \\
LAQuant       & 4  & 313.9 (x1.68) & 191.4 (x2.11) & 122.5 (x2.33) & 129.26 (x2.41) \\
ParoQuant  & 3  & 392.9 (x2.10) & 278.2 (x3.06) & 241.3 (x4.59) & 298.15 (x5.57) \\
LAQuant       & 3  & 468.7 (x2.50) & 325.0 (x3.58) & 276.9 (x5.27) & 345.47 (x6.45) \\
\bottomrule[1.2pt]
\end{tabular}
\end{table}

\begin{figure}[!ht]
  \centering
  \includegraphics[width=\linewidth]{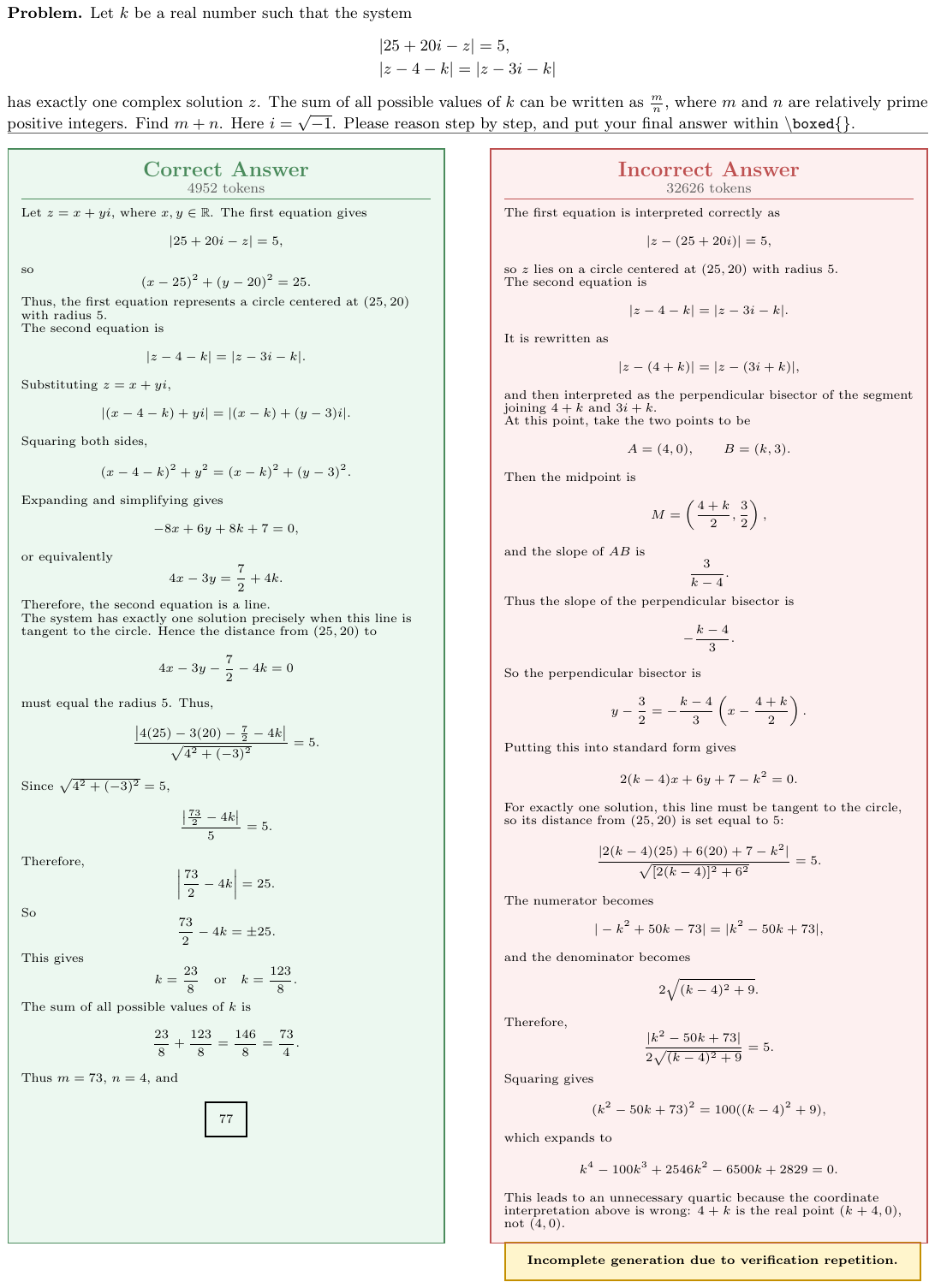}
  \caption{Comparison of generated solutions for AIME2025 Problem 7. The left panel shows the correct answer produced with LAQuant, while the right panel shows the incorrect ParoQuant output, which fails due to incomplete generation caused by verification repetition.}
  \label{fig:test}
\end{figure}
\section{Details of INT3 GEMV Kernel Modifications}
\label{app:int3_kernel}

The reference GPTQ INT3 GEMV kernel~\cite{gptq} is optimized for OPT-175B on A100 and 2$\times$A6000 configurations, employing a fixed thread block size of $\text{BLOCKWIDTH}=256$ and supporting only per-channel quantization. When applied to models in the Qwen3 and Llama families, this configuration yields suboptimal occupancy and lacks compatibility with both group-wise quantization and recent compilation infrastructure. We introduce three modifications to the kernel while preserving its original 3-bit packing format (32 weights encoded across three 32-bit words).

\paragraph{Templated thread block size.}
The reference kernel statically defines $\text{BLOCKWIDTH}=256$ input features per thread block, with a corresponding $\text{BLOCKHEIGHT}=24$ packed rows. For the matrix shapes encountered in small-to-medium models, this configuration produces an insufficient number of CUDA blocks to fully saturate modern GPUs. For instance, the $q_{\text{proj}}$ layer of Qwen3-1.7B (2048$\times$2048) issues a grid of only $8 \times 8 = 64$ blocks, corresponding to $0.6$ blocks per streaming multiprocessor (SM) on the A100. To address this, we template the kernel on $\text{BLOCKWIDTH} \in \{64, 128, 256\}$, with $\text{BLOCKHEIGHT}$ derived as $\text{BLOCKWIDTH}/32 \cdot 3$. For the matrix shapes in the Qwen3 family, $\text{BLOCKWIDTH}=128$ increases the block count by a factor of four and consistently yields the highest throughput.

\paragraph{Group-wise quantization.}
The published GPTQ kernel supports only per-channel quantization, in which a single $(\textit{scale}, \textit{zero})$ pair is shared across an entire output column. We extend the kernel to support group-wise quantization with $\text{groupsize}=128$ by storing the quantization parameters as a $(\textit{num\_groups}, \textit{outfeatures})$ tensor and reloading them at the beginning of each 32-element inner-loop iteration. As the inner loop processes exactly 32 input elements per iteration and the group size is a multiple of 32, group boundaries are guaranteed to align with iteration boundaries, ensuring that the additional control flow incurs negligible runtime cost. This modification brings the kernel into alignment with prevailing quantization recipes and improves accuracy at equivalent bit-widths.

\paragraph{Custom-op registration.}
The reference kernel is exposed only through a pybind11 binding, which prevents \texttt{torch.compile} with \texttt{fullgraph=True} from tracing through it. We register the kernel as a \texttt{torch.library.custom\_op} with a fake-tensor implementation, allowing the kernel to be captured into the compiled decode graph alongside the rest of the model.

\section{Additional Ablation Study}
\label{app:additional_abl_study}
We conduct ablation studies using Qwen3-1.7B under 3-bit quantization. Unless otherwise specified, we use 4,096 training samples, a maximum sequence length of 2,048, $k=1$ one-layer lookahead, and 20 training epochs per layer, varying only the target factor in each study.
\begin{table}[!ht]
\centering
\caption{Ablation studies on LAQuant design choices, including training set size and training sequence length.}
\label{tab:ablation_additional}

\begin{subtable}[!ht]{0.45\textwidth}
\centering
\caption{Training set size.}
\label{tab:ablation_data}
\scriptsize
\setlength{\tabcolsep}{3pt}
\renewcommand{\arraystretch}{0.95}
\begin{tabular}{lcccc}
\toprule
 & 1024 & 2048 & 4096 & 8192 \\
\midrule
AIME25 & 24.17 & 27.03 & 29.33 & 25.89 \\
LSAT   & 47.88 & 52.26 & 54.86 & 54.10 \\
Avg.   & 36.03 & 39.65 & \textbf{42.10} & 40.00 \\
\bottomrule
\end{tabular}
\end{subtable}
\hfill
\begin{subtable}[!ht]{0.45\textwidth}
\centering
\caption{Training sequence length.}
\label{tab:ablation_seqlen}
\scriptsize
\setlength{\tabcolsep}{3pt}
\renewcommand{\arraystretch}{0.95}
\begin{tabular}{lccc}
\toprule
 & 1024 & 2048 & 4096 \\
\midrule
AIME25 & 29.12 & 29.33 & 29.53 \\
LSAT   & 52.42 & 54.86 & 51.68 \\
Avg.   & 40.77 & \textbf{42.10} & 40.61 \\
\bottomrule
\end{tabular}
\end{subtable}
\end{table}
\paragraph{Training Set Size.}
As shown in Table~\ref{tab:ablation_data}, increasing the number of training samples consistently improves the average accuracy from 36.03 to 42.10 when the training set size is increased from 1,024 to 4,096. However, further increasing the size to 8,192 lowers the average accuracy to 40.00, indicating that the gain saturates beyond 4,096 samples. Therefore, we use 4,096 training samples as the default setting.

\paragraph{Training Sequence Length.}
As shown in Table~\ref{tab:ablation_seqlen}, increasing the sequence length from 1,024 to 2,048 improves the average accuracy. However, the gain saturates beyond 2,048 tokens, and further increasing the sequence length to 4,096 reduces the average accuracy to 40.61. Since longer sequences also increase training cost, 2,048 provides the best trade-off and is used as the default maximum sequence length.

\begin{table}[!ht]
\centering
\caption{Effect of training-token segment location on reasoning performance. Values denote Pass@1$^{\uparrow}$ on five reasoning benchmarks and their average.}
\label{tab:segment_ablation}
\scriptsize
\setlength{\tabcolsep}{4pt}
\renewcommand{\arraystretch}{0.85}
\begin{tabular}{llcccccc}
\toprule[1.2pt]
\textbf{Location} & \textbf{Bits} & \textbf{AIME25} & \textbf{GPQA} & \textbf{LSAT} & \textbf{MMLU-Pro} & \textbf{LCB} & \textbf{Avg.} \\
\midrule
& 16 & 35.94 & 40.70 & 64.10 & 57.31 & 34.33 & 46.48 \\
\midrule
Prefix-2k          & 3 & 29.33 & 31.79 & 54.86 & 47.94 & 21.64 & \textbf{37.11} \\
Response-2k & 3 & 28.54 & 31.90 & 51.14 & 46.82 & 19.40 & 35.56 \\
Middle-2k   & 3 & 29.79 & 31.50 & 50.98 & 47.02 & 19.40 & 35.74 \\
Suffix-2k     & 3 & 26.88 & 31.09 & 53.86 & 46.82 & 20.52 & 35.83 \\
\midrule
Prefix-2k   & 4 & 35.47 & 36.43 & 61.14 & 55.98 & 32.46 & \textbf{44.30} \\
Response-2k & 4 & 34.48 & 36.40 & 60.68 & 55.07 & 30.97 & 43.52 \\
Middle-2k   & 4 & 33.80 & 34.40 & 58.61 & 55.59 & 33.96 & 43.27 \\
Suffix-2k     & 4 & 34.43 & 36.99 & 58.78 & 55.44 & 29.85 & 43.10 \\
\bottomrule[1.2pt]
\end{tabular}
\end{table}
\paragraph{Training-Token Segment Location.}
We further examine whether the location of the training tokens within each chat-templated prompt--response sequence affects quantization performance, using a 2k-token segment, where 2k denotes 2,048 tokens. Specifically, \emph{Prefix-2k} uses the first 2k tokens of the full sequence, \emph{Response-2k} uses the first 2k tokens after \texttt{<im\_start>assistant}, \emph{Middle-2k} randomly samples a 2k-token span from the middle region after excluding the first and last 2k tokens, and \emph{Suffix-2k} uses the last 2k tokens. As shown in Table~\ref{tab:segment_ablation}, Prefix-2k achieves the highest average accuracy under both W3G128 and W4G128. 

The results suggest that aligning the training segment with the token sequence observed during actual autoregressive decoding is important for reasoning-domain QAT. Although the response tokens are the primary generation target, the prompt tokens remain in the KV cache throughout decoding and continuously participate in attention computation. Therefore, training on the prefix segment, which contains both the prompt and the early response context, better matches the inference-time computation than using only response-side or later tokens. This helps preserve the prompt-conditioned KV cache and provides more effective calibration signals for long-form reasoning generation.
\section{Additional Experiments}
\begin{table}[!ht]
\centering
\caption{Effect of direct K-cache supervision on E2E QAT. Pass@1 results are reported for the vanilla E2E baseline and E2E with an additional per-layer K-cache RMSE loss.}
\label{tab:e2e_k_rmse}
\scriptsize
\setlength{\tabcolsep}{2pt}
\renewcommand{\arraystretch}{0.75}
\begin{tabular}{lcccccc}
\toprule[1.2pt]
\textbf{Method} & \textbf{AIME25} & \textbf{GPQA} & \textbf{LSAT} & \textbf{MMLU-Pro} & \textbf{LCB} & \textbf{Avg.} \\
\midrule
E2E           & 6.62 & 28.28 & 26.14    & 29.98    & 4.48   & 19.10 \\
E2E w/ K-cache RMSE    & 8.60 & 29.86 & 27.69 & 31.77 & 5.60 & 20.70 \\
\bottomrule[1.2pt]
\end{tabular}
\end{table}
\paragraph{Effect of Direct K-Cache Supervision.}
Table~\ref{tab:e2e_k_rmse} provides an additional controlled experiment supporting our gradient-direction analysis. Our framework suggests that E2E supervision can attenuate the learning signal for residual directions that are important for preserving the next-layer K-cache, because these directions lie in a weakly responsive subspace of the E2E-induced metric. To test this implication, we augment the E2E objective with an auxiliary K-cache RMSE loss at each layer, thereby providing direct supervision on the key states. The resulting improvement over the vanilla E2E baseline indicates that explicitly preserving K-cache fidelity can partially recover reasoning performance. This supports our claim that the standard E2E objective alone does not sufficiently emphasize K-cache-preserving directions.
\section{Licenses and Terms of Use for Existing Assets}
\label{app:licenses}
\paragraph{Licenses for existing assets.}
Table~\ref{tab:asset_licenses} summarizes the licenses and terms of use for the existing assets used in this work, including models, training datasets, and evaluation benchmarks. We use these assets only for research purposes and follow the corresponding licenses, terms of use, and attribution requirements. 

\begin{table}[!ht]
\centering
\small
\caption{Licenses and terms of use for existing assets used in this work.}
\label{tab:asset_licenses}
\resizebox{\textwidth}{!}{%
\begin{tabular}{lll}
\toprule
Asset & Type & License / Terms of Use \\
\midrule
WikiText-2 & Dataset & CC BY-SA 4.0 \\
C4 & Dataset & ODC-BY; subject to Common Crawl Terms of Use \\
RedPajama & Dataset & Subset-dependent licenses; Common Crawl Terms of Use for web data \\
AIME24 & Benchmark & Apache-2.0 for the dataset source used \\
AIME25 & Benchmark & Apache-2.0 for the dataset source used \\
GPQA-Diamond & Benchmark & CC BY 4.0 \\
LSAT-AR & Benchmark & MIT \\
LiveCodeBench & Benchmark & MIT for the benchmark toolkit; original platform terms apply \\
OpenR1-Math-220k & Dataset & Apache-2.0 \\
OpenThoughts3 & Dataset & Apache-2.0 \\
MMLU-Pro & Benchmark & MIT \\
ARC-C / ARC-E & Benchmark & CC BY-SA 4.0 \\
BoolQ & Benchmark & CC BY-SA 3.0 \\
HellaSwag & Benchmark & MIT \\
Qwen3 series & Model & Apache-2.0 \\
Llama-3.1-8B & Model & Llama 3.1 Community License \\
DeepSeek-R1-Distill-Llama-8B & Model & MIT; derived from Llama-3.1-8B-Base under the Llama 3.1 license \\
\bottomrule
\end{tabular}%
}
\end{table}

\section{Limitations and Future Work}
\label{app:limitations}

LAQuant has several limitations worth acknowledging. At W3G128, a measurable gap to FP16 remains on competitive reasoning: 9.67, 8.28, and 5.53\,pp on the Qwen3 1.7B/4B/8B reasoning average (Table~\ref{tab:res_main_reasoning}). The gap shrinks with model scale, suggesting that low-bit reasoning quantization is hardest for smaller models, where each weight carries proportionally more task-relevant information.

Although Section~\ref{sec:calibdata} formalizes Hessian-subspace alignment as the diagnostic for matching calibration to deployment, the corpus we use (DeepSeek-R1 traces over OpenR1-Math-220k) was chosen empirically and validated post-hoc against the diagnostic. We do not provide a procedure for selecting or synthesizing calibration data given a target deployment distribution. The approach also implicitly assumes that a teacher trace generator (DeepSeek-R1 in our experiments) is available for the deployment domain; in domains lacking such a generator, calibration falls back to general text and the alignment guarantee weakens.

Compute constraints also limited model-scale coverage. We evaluate up to 8B parameters, matching ReasoningQAT's coverage but stopping one model size short of ParoQuant, whose original paper reports up to Qwen3-14B. Larger open reasoning models (DeepSeek-R1-distilled 70B, Qwen3-30B-A3B, Qwen3-235B-A22B) exceed our compute budget. Layer-wise QAT cost scales linearly with parameter count, so the method is tractable in principle, but verifying that the gradient-direction analysis transfers at scale is left to future work. The lookahead loss itself adds a training-time memory cost: the teacher's $\ell$ and $\ell{+}1$ activations must be held simultaneously, raising peak VRAM relative to plain layer-wise QAT (Table~\ref{tab:res_cost}).

Several extensions are natural follow-ups. Extending LAQuant to W$b$A$b$KV$b$ regimes would compound the inference benefits and could augment the implicit Jacobian weighting with an explicit K- and V-cache reconstruction term that directly targets the cache-writing projections rather than relying on the implicit attribution alone. The principal-angle diagnostic of Section~\ref{sec:calibdata} could likewise be operationalized as a constructive procedure: for example, greedy maximization of $\bar c_k(\mathcal{D}_{\mathrm{cal}}, \mathcal{D}_{\mathrm{eval}})$ over a candidate corpus, or Hessian-aware prompt synthesis when no suitable corpus exists. Quantization of mixture-of-experts models (e.g., Qwen3-30B-A3B, Qwen3-235B-A22B) is another open direction: per-expert calibration coverage interacts non-trivially with the Hessian-alignment diagnostic, since each expert's activation distribution depends on the router's pattern under the calibration corpus. Finally, Table~\ref{tab:method_ablation_lookahead} reports $k{=}1$ as the best uniform lookahead depth, but layers vary in downstream-Jacobian conditioning; allowing $k$ to vary per layer (e.g., based on a cheap spectral proxy of $J_{\ell+1}$) could recover the marginal gain that uniform $k{=}2$ does not deliver.

\newpage

\end{document}